\documentclass[final,leqno,onefignum,onetabnum]{siamltexmm}
 
\usepackage{amssymb}
\usepackage{amsfonts}
\usepackage{amsmath}
\usepackage{epsfig}
\usepackage{subfigure}
\usepackage{algorithm}
\usepackage{algorithmic}
\usepackage{array}
\usepackage{tabularx}
\usepackage{multirow}
\usepackage{soul}

\title{Boosting of Image Denoising Algorithms\thanks{This research was supported by the European Research Council under EU's 7th Framework Program, ERC Grant agreement no. 320649, by the Intel Collaborative Research Institute for Computational Intelligence, and by Google Faculty Research Award.
}}

\author{Yaniv~Romano\footnotemark[2]\
\and Michael~Elad\footnotemark[3]}

\begin{document}
\maketitle
\newcommand{\slugmaster}{}

\renewcommand{\thefootnote}{\fnsymbol{footnote}}
\footnotetext[2]{Department of Electrical Engineering, Technion -- Israel Institute of Technology, Technion City, Haifa 32000, Israel (yromano@tx.technion.ac.il).}
\footnotetext[3]{Department of Computer Science, Technion -- Israel Institute of Technology, Technion City, Haifa 32000, Israel (elad@cs.technion.ac.il).}

\renewcommand{\thefootnote}{\arabic{footnote}}

\begin{abstract}
\label{abstract}
In this paper we propose a generic recursive algorithm for improving image denoising methods. Given the initial denoised image, we suggest repeating the following ''SOS'' procedure: (i) \emph{(S)trengthen} the signal by adding the previous denoised image to the degraded input image, (ii) \emph{(O)perate} the denoising method on the strengthened image, and (iii) \emph{(S)ubtract} the previous denoised image from the restored signal-strengthened outcome. The convergence of this process is studied for the K-SVD image denoising and related algorithms. Still in the context of K-SVD image denoising, we introduce an interesting interpretation of the SOS algorithm as a technique for closing the gap between the local patch-modeling and the global restoration task, thereby leading to improved performance. In a quest for the theoretical origin of the SOS algorithm, we provide a graph-based interpretation of our method, where the SOS recursive update effectively minimizes a penalty function that aims to denoise the image, while being regularized by the graph Laplacian. We demonstrate the SOS boosting algorithm for several leading denoising methods (K-SVD, NLM, BM3D, and EPLL), showing tendency to further improve denoising performance.
\end{abstract}



\pagestyle{myheadings}
\thispagestyle{plain}
\markboth{YANIV~ROMANO AND MICHAEL~ELAD}{BOOSTING~OF~IMAGE~DENOISING~ALGORITHMS}

\def \x{{\mathbf x}}
\def \v{{\mathbf v}}
\def \xh{{\hat{\x}}}
\def \xi{\x_i}

\def\oneb{\mathrm{1}}
\def\one{\underbar{$ \oneb $}}
\def\y{{\mathbf y}}
\def\z{{\mathbf z}}
\def\p{{\mathbf p}}
\def\q{{\mathbf q}}
\def\U{{\mathbf U}}

\def\Gam{{\mathbf \Gamma}}
\def\Q{{\mathbf Q}}
\def\D{{\mathbf D}}
\def\W{{\mathbf W}}
\def\Z{{\mathbf Z}}
\def\Zh{{\mathbf \hat \Z }}
\def\V{{\mathbf V}}
\def\S{{\mathbf S}}
\def\G{{\mathbf G}}
\def\Dh{{\hat{\D}}}

\def\L{{\mathrm L}}
\def\I{{\mathrm I}}
\def\b{{\mathrm b}}
\def\W{{\mathbf W}}
\def\A{{\mathbf A}}
\def\B{{\mathrm B}}
\def\Ai{\A_i}
\def\T{{\mathrm T}}
\def\Ti{\T_i}
\def\R{{\mathbf R}}
\def\r{{\mathrm r}}
\def\c{{\mathrm c}}
\def\hhsigma{{\hat{\sigma}}}
\def\RR{{\mathbb R}}

\section{Introduction}
\label{intro}
Image denoising is a fundamental restoration problem. Consider a given measurement image $ \y \in \RR^{r\times c} $, obtained from the clean signal $ \x \in \RR^{r\times c} $ by a contamination of the form
\begin{align}
\label{deg_model}
\y = \x + \v,
\end{align}
where $ \v \in \RR^{r\times c} $ is a zero-mean additive noise that is independent with respect to $ \x $. Note that $\x$ and $\y$ are held in the above equation as column vectors after lexicographic ordering. A solution to this inverse problem is an approximation $\hat{\x}$ of the unknown clean image $\x$. 

Plenty of sophisticated algorithms have been developed in order to estimate the original image content, 
the NLM \cite{NL_DENOISE_REF4}, K-SVD \cite{KSVD_REF1}, BM3D \cite{BM3D_REF}, EPLL \cite{zoran2011learning}, and others \cite{PLE_REF, burger2012image, ram2013image, NL_DENOISE_REF3, talebi2013saif, romanoimproving}. These algorithms rely on powerful image models/priors, where sparse representations \cite{SPARSE_REF1,SPARSE_REF2} and processing of local patches \cite{lebrun2012secrets} have become two prominent ingredients.

Despite the effectiveness of the above denoising algorithms, improved results can be obtained by applying a boosting technique (see \cite{talebi2013saif,charest2006general, moderntour} for more details). There are several such techniques that were proposed over the years, e.g. ''twicing'' \cite{tukey1977exploratory}, Bregman iterations \cite{osher2005iterative}, $ l_2 $-boosting \cite{buhlmann2003boosting}, SAIF \cite{talebi2013saif} and more (e.g. \cite{romanoimproving}). These algorithms are closely related and share in common the use of the residuals (also known as the ''method-noise'' \cite{NL_DENOISE_REF4}) in order to improve the estimates. The residual is defined as the difference between the noisy image and its denoised version. Naturally, the residual contains signal leftovers due to imperfect denoising (together with noise). 

For example, motivated by this observation, the idea behind the twicing technique \cite{tukey1977exploratory} is to extract these leftovers by denoising the residual, and then add them back to the estimated image. This can be expressed as \cite{charest2006general}
\begin{align}
\label{twicing_algo}
\xh^{k+1} = \xh^{k} + f\left( \y - \xh^{k} \right),
\end{align}
where the operator $ f\left( \cdot \right)$ represents the denoising algorithm and $ \xh^{k} $ is the $ k^{th} $ iteration denoised image. The initialization is done by setting $ \xh^{0} = \textbf{0}$.

Using the concept of Bregman distance \cite{bregman1967relaxation} in the context of total-variation denoising \cite{rudin1992nonlinear}, Osher et al. \cite{osher2005iterative} suggest exploiting the residual by
\begin{align}
\label{bregman_algo}
\xh^{k+1} = f\left( \y + \sum_{i=1}^{k}{ \left( \y - \xh^{i}\right)  }\right),
\end{align}
where the recursive function is initialized by setting $ \xh^{0} = \textbf{0}$. Note that if the denoising algorithm $ f\left( \cdot \right)$ can be represented as a linear (data-independent) matrix, Equations (\ref{twicing_algo}) and (\ref{bregman_algo}) coincide \cite{charest2006general}. Furthermore, for these two boosting techniques, it has been shown \cite{talebi2013saif} that as $ k $ increases, the estimate $ \xh^{k} $ returns to the noisy image $ \y $.

Motivated by the above-mentioned algorithms, our earlier work \cite{romanoimproving} improves the K-SVD \cite{KSVD_REF1}, NLM \cite{NL_DENOISE_REF4} and the first-stage of the BM3D \cite{BM3D_REF} by applying an iterative boosting algorithm that extracts the ''stolen'' image content from the method-noise image. The improvement is achieved by adding the extracted content back to the initial denoised result. The work in \cite{romanoimproving} suggests representing the signal leftovers of the method-noise patches using the same basis/ support that was chosen for the representation of the corresponding clean patch in the initial denoising stage. As an example, in the context of the K-SVD, the supports are sets of atoms that participate in the representation the noisy patches.

However, in addition to signal leftovers that reside in the residual image, there are noise leftovers that are found in the denoised image. Driven by this observation, SAIF \cite{talebi2013saif} offers a general framework for improving spatial domain denoising algorithms. Their algorithm controls the denoising strength locally by filtering iteratively the image patches. Per each patch, it chooses automatically the improvement mechanism: twicing or diffusion, and the number of iterations to apply. The diffusion \cite{moderntour} is a boosting technique that suggests repeating applications of the same denoising filter, thereby removing the noise leftovers that rely in the previous estimate (sometimes also sacrificing some of the high-frequencies of the signal).

In this paper we propose a generic recursive function that treats the denoising method as a ''black-box'' and has the ability to push it forward to improve its performance. Differently from the above methods, instead of adding the residual (which mostly contains noise) back to the noisy image, or filtering the previous estimate over and over again (which could lead to over-smoothing), we suggest \emph{strengthening the signal} by leveraging on the availability of the denoised image. More specifically, given an initial estimation of the cleaned image, improved results can be achieved by repeating iteratively the following SOS procedure:
\begin{remunerate}
\item \emph{Strengthen} the signal by adding the previous denoised image to the noisy input image.
\item \emph{Operate} the denoising method on the strengthened image.
\item \emph{Subtract} the previous denoised image from the restored signal-strengthened outcome.
\end{remunerate}
The core equation that describes this procedure can be written in the following form: 
\begin{align}
\label{sos_algo}
\xh^{k+1} = f\left( \y + \xh^{k}\right) - \xh^k,
\end{align}
where $ \xh^{0} = \textbf{0}$. As we show hereafter, a performance improvement is achieved since the signal-strengthened image can be denoised more effectively compared to the noisy input image, due to the improved Signal to Noise Ratio (SNR).

The convergence of the proposed algorithm is studied in this paper by formulating the linear part of the denoising method and assessing the iterative system's matrix properties. In this work we put special emphasis on the K-SVD and describe the resulting denoising matrix and the corresponding convergence properties related to it. The work by Milanfar \cite{moderntour} shows that most existing denoising algorithms (e.g. NLM \cite{NL_DENOISE_REF4}, Bilateral filter \cite{tomasi1998bilateral}, LARK \cite{chatterjee2012patch}) can be represented as a row-stochastic positive definite matrices. In this context, our analysis suggests that for most denoising algorithms, the proposed SOS boosting method is guaranteed to converge. {\color{black}{Therefore, we get a straightforward stopping criterion.}} 

In addition, we introduce an interesting interpretation of the SOS boosting algorithm, related to a major shortcoming of patch-based methods: the gap between the local patch-processing and the global need for a whole restored image. In general, patch-based methods (i) break the image into overlapping patches, (ii) restore each patch (local processing), and (iii) reconstruct the image by aggregating the overlapping patches (the global need). The aggregation is usually done by averaging the overlapping patches. The proposed SOS boosting is related to a different algorithm that aims to narrow the local-global gap mentioned above \cite{romanosharing}. Per each patch, this algorithm defines the difference between the local (intermediate) result and the patch from the global outcome as a ''disagreement''. Since each patch is processed independently, such disagreement naturally exists.

Interestingly, in the context of the K-SVD image denoising, the SOS algorithm is equivalent to repeating the following steps (see \cite{romanosharing} and Section \ref{sharing_sec} for more details): (i) compute the disagreement per patch, (ii) subtract the disagreement from the degraded input patches, (iii) apply the restoration algorithm on these patches, and (iv) reconstruct the image. Therefore, the proposed algorithm encourages the overlapping patches to share their local information, thus reducing the gap between the local patch-processing and the global restoration task.

The above should remind the reader of the EPLL framework \cite{zoran2011learning}, which also addresses the local-global gap. EPLL encourages the patches of the \emph{final image} (i.e. after patch-averaging) to comply with the local prior. In EPLL, given a local patch model, the algorithm alternates between denoising the previous result according to the local prior, followed by an image reconstruction step (patch-averaging). Several local priors can use this paradigm -- Gaussian Mixture Model (GMM) is suggested in the original paper \cite{zoran2011learning}. Similarly, EPLL with sparse and redundant representation modeling has been recently proposed in \cite{jereepll}. EPLL bares some resemblance to diffusion methods \cite{moderntour}, as it amounts to iterated denoising with a diminishing variance setup, in order to avoid an over-smoothed outcome. In practice, at each diffusion step, the authors of \cite{zoran2011learning,jereepll} empirically estimate the noise that resides in $ \xh^k $ (which is neither Gaussian nor independent of $ \xh^k $). In contrast, in our scheme, setting this parameter is trivial -- the noise level of $ \y+\xh^k $ is nearly $ \sigma $, regardless of the iteration number.

In the context of image denoising, several works (e.g. \cite{elmoataz2008nonlocal,bougleux2009local,PeymanLaplaceDenoising,PeymanLaplace,symm}) suggest representing an image as a weighted graph, where the weights measure the similarity between the pixels/patches. Since the graph Laplacian describes the structure of the underlying signal, it can be used as an adaptive regularizer, as done in the above-mentioned methods. Put differently, the graph Laplacian preserves the geometry of the image by promoting similar pixels to remain similar, thus achieving an effective denoising performance. It turns out that the steady-state outcome of the SOS minimizes a cost function that involves the graph Laplacian as a regularizer, providing another justification for the success of our method. Furthermore, influenced by the SOS mechanism, we offer novel iterative algorithms that minimize the graph Laplacian cost functions that are defined in \cite{elmoataz2008nonlocal,bougleux2009local,PeymanLaplaceDenoising}. Similarly to the SOS, the proposed iterative algorithms treat the denoiser as a ''black-box'' and operate on the strengthened image, without an explicit construction of the weighted graph.

This paper is organized as follows: In Section \ref{background} we provide brief background material on sparse representation and dictionary learning, with a special attention to the K-SVD denoising and its matrix form. In Section \ref{poposed_algo} we introduce our novel SOS boosting algorithm, study its convergence, and generalize it by introducing two parameters that govern the steady-state outcome, the requirements for convergence and the rate-of-convergence. In Section \ref{disagreement_algo} we discuss the relation between the SOS boosting and the local-global gap. In Section \ref{graph} we provide a graph-based analysis to the steady-state outcome of the SOS, and offer novel recursive algorithms for related graph Laplacian methods. Experiments are brought in Section \ref{experiments}, showing a meaningful improvement of the K-SVD image denoising, and similar boosting for other methods -- the NLM, BM3D, and EPLL. Conclusions and future research directions are drawn in Section \ref{conclusions}.

\section{K-SVD Image Denoising Revisited}
\label{background}

We bring the following discussion on sparse representations and specifically the K-SVD image denoising algorithm, because its matrix interpretation will serve hereafter as a benchmark in the convergence analysis.

\subsection{Sparse Representation \& K-SVD Denoising}

The sparse-land modeling \cite{SPARSE_REF1,SPARSE_REF2} assumes that a given signal $x\in\RR^n$ (in this context, the signal $ x $ is not necessarily an image) can be well represented as $ x=\D\alpha $, where $ \D\in\RR^{n \times m} $ is a dictionary composed of $m \ge n$ atoms as its columns, and $ \alpha\in\RR^m $ is a sparse vector, i.e, has a few non-zero coefficients. For a noisy signal $ y = x + v$, we seek a representation $ \hat{\alpha} $ that approximates $ x $ up to an error bound, which is proportional to the amount of noise in $ v $. This is an NP-hard problem that can be expressed as
\begin{align}
\label{background_sp}
\hat{\alpha}  = \min_{\alpha} {\Arrowvert \alpha \Arrowvert}_0  \hspace{0.5em} \mathrm{s.t.} \hspace{0.5em}  \Arrowvert\D \alpha-y\Arrowvert_{2}^{2}\le\epsilon^2,
\end{align}
where $\Arrowvert\alpha\Arrowvert_0$ counts the non-zero coefficients in $\alpha$, and the constant $\epsilon$ is an error bound. There are many efficient sparse-coding algorithms that approximate the solution of Equation (\ref{background_sp}), such as OMP \cite{OMP_REF}, BP \cite{BASIS_REF}, and others \cite{SPARSE_REF2, tropp2010computational}.

The above discussion assumes that $ \D $ is known and fixed. A line of work (e.g. \cite{MOD_REF, smith2013improving, KSVD_REF2}) shows that adapting the dictionary to the input signal results in a sparser representation. In the case of denoising, under an error constraint, since the dictionary is adapted to the image content, the subspace that the noisy signal is projected onto is of smaller dimension, compared to the case of a fixed dictionary. This leads to a stronger noise reduction, i.e, better restoration. Given a set of measurements $ \{\y_i\}_{i=1}^{N} $, a typical dictionary learning process \cite{KSVD_REF2,MOD_REF} is formulated as
\begin{align}
\label{dictionarylearning}
\left[ {\Dh}, \{\hat \alpha_i\}_{i=1}^{N} \right]  =  \min_{\D, \{\alpha_i\}_{i=1}^{N}} \sum_{i=1}^{N} \gamma_i\|\alpha_i\|_0 + \|\D\alpha_i - y_i\|_2^2, 
\end{align}
where $ \Dh $ and $ \{\hat \alpha_i\}_{i=1}^{N} $ are the resulting dictionary and representations, respectively. The scalars $ \gamma_i $ are signal dependent, so as to comply with a set of constraints of the form $ \|\D\alpha_i - y_i\|_2^2 \le \epsilon^2 $.

Due to computational demands, adapting a dictionary to large signals (images in our case) is impractical. Therefore, a treatment of an image is done by breaking it into overlapping patches (e.g. of size $ 8 \times 8 $). Then, each patch is restored according to the sparsity-inspired prior. More specifically, the K-SVD image denoising algorithm \cite{KSVD_REF1} divides the noisy image into $\sqrt{n} \times \sqrt{n}$ fully overlapping patches, then processes them \emph{locally} by performing iterations of sparse-coding (using OMP) and dictionary learning as described in Equation (\ref{dictionarylearning}). Finally, the \emph{global} denoised image is obtained by returning the cleaned patches to their original locations, followed by an averaging with the input noisy image. The above procedure approximates the solution of
\begin{align}
\label{KSVD_func}
\left[ {\xh},{\Dh}, \{\hat \alpha_i\}_{i=1}^{N} \right]  =
 \min_{\x,\D, \{\alpha_i\}_{i=1}^{N}}  \mu \|\x-\y \|_2^2 
 + \sum_{i=1}^{N} \gamma_i\|\alpha_i\|_0 + \|\D\alpha_i - \R_i \x\|_2^2, 
\end{align}
where $\xh \in\RR^{rc}$ is the resulting denoised image, $ N $ is the number of patches, and $\R_i\in\RR^{n \times rc}$ is a matrix that extracts the $i^{th}$ patch from the image. The first term in Equation (\ref{KSVD_func}) demands a proximity between the noisy and denoised images. The second term demands that each patch $ \R_i\x $ is represented sparsely up to an error bound, with respect to a dictionary $ \D $. As to the coefficients $ \gamma_i $, those are spatially dependent, and set as explained in Equation ($ \ref{dictionarylearning} $).

\subsection{K-SVD Image Denoising: A Matrix Formulation}
\label{matform}

The K-SVD image denoising can be divided into non-linear and linear parts. The former is composed of preparation steps that include the support determination within the sparse-coding and the dictionary update, while the outcome of the latter is the actual image-adaptive filter that cleans the noisy image. The matrix formulation of the K-SVD denoising represents its linear part, assuming the availability of the non-linear computations. At this stage we should note that \emph{the following formulation is given as a background to the theoretical analysis that follows, and it is not necessary when using the proposed SOS boosting in practice}.

Sparse-coding determines per each noisy patch $ \R_i\y $ a small set of atoms $ \D_{s_i} $ that participate in its representation. Following the last step of the OMP \cite{OMP_REF}, given $ \D_{s_i} $, the representation\footnote{We abuse notations here as $ \alpha_i $ refers hereafter only to the non-zero part of the representation, being a vector of length $ |s_i| \ll m $.} $ \alpha_i $ of the clean patch is obtained by solving
\begin{align}
\label{denoise_patch_min}
\alpha_i  = \min_{z} \| \D_{s_i} z - \R_i\y \|_2^2,
\end{align}
which has a closed-form solution
\begin{align}
\label{denoise_patch_closed_form}
\alpha_i = (\D_{s_i}^{T}\D_{s_i})^{-1}\D_{s_i}^{T} \R_i\y. 
\end{align}
Given $ \alpha_i $, the clean patch $ \hat{\p}_i $ is obtained by applying the inverse transform from the representation to the signal/ patch space, i.e., 
\begin{align}
\label{clean_patch}
\hat{\p}_i & = \D_{s_i}\alpha_i \\ \notag
     & = \D_{s_i}(\D_{s_i}^{T}\D_{s_i})^{-1}\D_{s_i}^{T} \R_i\y. 
\end{align}
Notice that although the computation of $ s_i $ is non-linear, the clean patch $ \hat{\p}_i $ is obtained by filtering its noisy version, $ \R_i\y $, with a linear, image-adaptive, symmetric and normalized filter. 

Following Equation (\ref{KSVD_func}) and given all $ \hat{\p}_i = \D_{s_i}\alpha_i $, the globally denoised image $ \xh $ is obtained by minimizing
\begin{align}
\label{denosied_image_optimization}
{\xh} =
\min_{\x} \text{   } \mu \|\x-\y \|_2^2 + \sum_{i=1}^{N} \|\hat{\p}_i - \R_i \x\|_2^2. 
\end{align}
This is a quadratic expression that has a closed-form solution of the form
\begin{align}
\label{denosied_image_filtered}
\xh & { = \left( \mu \I + \sum_{i=1}^{N} \R_i^T\R_i \right)^{-1} \left( \mu \y + \sum_{i=1}^{N} \R_i^T\hat{\p}_i \right)}   \\ \notag
    & { = \left( \mu \I + \sum_{i=1}^{N} \R_i^T\R_i \right)^{-1} \left( \mu \I + \sum_{i=1}^{N} \R_i^T \D_{s_i}(\D_{s_i}^{T}\D_{s_i})^{-1}\D_{s_i}^{T} \R_i \right) \y } \\ \notag
    & { = \boldsymbol{D}^{-1}\boldsymbol{K}\y} \\ \notag
    & { = \W\y},
\end{align}
where $ \I \in \RR^{rc \times rc}$ is the identity matrix. The term $ \sum_{i=1}^{N} \R_i^T\R_i $ is a diagonal matrix that counts the appearances of each pixel (e.g. 64 for patches of size $ 8\times8 $) and $ \mu\I $ originates from the averaging with the noisy image $ \y $. The matrix $ \R_i^{T} $ returns a clean patch $ \hat{\p}_i $ to its original location in the global image. The matrix $ \W \in \RR^{rc \times rc} $ is the resulting filter matrix formulation of the linear part of the K-SVD image denoising. In the context of graph theory, $ \boldsymbol{D} $ and $ \boldsymbol{K} $ are called the degree and similarity matrices, respectively (see Section \ref{graph} for more information).

A series of works \cite{talebi2013saif, moderntour, talebi2014global} studies the algebraic properties of such formulations for several image denoising algorithms (NLM \cite{NL_DENOISE_REF4}, Bilateral filter \cite{tomasi1998bilateral}, Kernel Regression \cite{chatterjee2012patch}), for which the filter-matrix is non-symmetric and row-stochastic matrix. Thus, this matrix has real and positive eigenvalues in the range of $ [0, 1] $, and the largest eigenvalue is unique and equals to $ 1 $, with a corresponding eigenvector $ [1, 1, ... ,1]^T $ \cite{seneta1981springer, horn2012matrix}. In the K-SVD case, and under the assumption of periodic boundary condition\footnote{See Appendix \ref{periodic} for an explanation on this requirement.}, the properties of the resulting matrix somewhat different, and are given in the following theorem.
\pagebreak

\begin{theorem}
\label{PropertiesofW}
The resulting matrix $ \W $ has the following properties:
\begin{enumerate}
\item Symmetric $ \W=\W^T $, and thus all eigenvalues are real.
\item Positive definite $ \W \succ 0 $, and thus all eigenvalues are strictly positive.
\item Minimal eigenvalue of $ \W $ satisfy $ \lambda_{min}(\W) \geq \frac{\mu}{\mu + n}$, where $ n $ is the patch size.
\item Doubly stochastic, in the sense of $ \W \one = \W^T\one = \one $. Note that $ \W $ may have negative entries, which violates the classic definition of row- or column- stochasticity.
\item The above implies that $ 1 $ is an eigenvalue corresponding to the eigenvector $ \one $.
\item The spectral radius of $ \W $ equals to $ 1 $, i.e, $ \|\W\|_2=1 $.
\item The above implies that maximal eigenvalue satisfy $ \lambda_{max}(\W)=1 $.
\item The spectral radius $ \|\W-\I\|_2\leq\frac{n}{\mu + n}<1 $.
\end{enumerate}
\end{theorem}
\noindent Appendix \ref{Wproperties} provides a proof for these claims.

For the denoising algorithms studied in \cite{talebi2013saif, moderntour, talebi2014global}, the matrix $ \W $ is not symmetric nor positive definite, however it can be approximated as such using the Sinkhorn procedure \cite{moderntour}. In the context of the K-SVD, as describe in Appendix \ref{periodic}, $ \W $ can become symmetric by a proper treatment of the boundaries (essentially performing cyclic processing of the patches).

To conclude, the discussion above shows that the K-SVD is a member in a large family of denoising algorithms that can be represented as matrices \cite{moderntour}. We will use this formulation in order to study the convergence of the proposed SOS boosting and for demonstrating the local-global interpretation.

\section{SOS Boosting} 
\label{poposed_algo}

In this section we describe the proposed algorithm, study its convergence, and generalize this algorithm by introducing two parameters that govern its steady-state outcome, the requirements for convergence and its rate.

\subsection{SOS Boosting - The Core Idea}
Leading image/patch priors are able to effectively distinguish the signal content from the noise. However, an emphasis of the signal over the noise could help the prior to better identify the image content, thereby leading to better denoising performance. As an example, the sparsity-based K-SVD could choose atoms that better fit the underlying signal. Similarly, the NLM, which cleans a noisy patch by applying a weighted average with its spatial neighbors, could determine better weights. This is the key idea behind the proposed SOS boosting algorithm, which exploits the previous estimation in order to enhance the underlying signal. In addition, the proposed algorithm treats the denoiser as a ''black-box'', thus it is easy to use and becomes applicable to a wide range of denoising methods.

As mentioned in Section \ref{intro}, the first class of boosting algorithms (twicing \cite{tukey1977exploratory} or its variants \cite{osher2005iterative, buhlmann2003boosting, romanoimproving}) suggest extracting the ''stolen'' content from the method-noise image, with the risk of returning noise back to the denoised image, together with the extracted information. On the other hand, the second class of boosting methods (diffusion \cite{moderntour} or EPLL \cite{zoran2011learning, jereepll}) aim at removing the noise that resides in the estimated image, with the risk of obtaining an over-smoothed result (this depends on the number of iterations or the denoiser parameters at each iteration). As a consequence, these two classes of boosting algorithms are somewhat lacking as they address only one kind of leftovers \cite{talebi2013saif} -- the one that reside in the method-noise or the other which is found in the denoised image. Also, these methods may result in under- or over- smoothed version of the noisy image.

Adopting a different perspective, we suggest \emph{strengthening} the signal by adding the clean image $ \xh^k $ to the noisy input $ \y $, and then \emph{operating} the denoising algorithm on the strengthened result. Differently from diffusion filtering, as the estimated part of the signal is emphasized, there is no loss of signal content that has not been estimated correctly (due to the availability of $ \y $). Differently from twicing, we hardly increase the noise level (under the assumption that the energy of the noise which resides in the clean image is small). Finally, a \emph{subtraction} of $ \xh^k $ from the outcome should be done in order to obtain a faithful denoised result. This procedure is formulated in Equation (\ref{sos_algo}):
\begin{align*}
\xh^{k+1} = f\left( \y + \xh^{k}\right) - \xh^k,
\end{align*}
where $ \xh^{0} = \textbf{0} $. 

The SOS boosting obtains improved denoising performance due to higher SNR of the signal-strengthened image, compared to the noisy input. In order to demonstrate this, let us denote
\begin{align}
\xh = \x + \v_r,
\end{align}
where $ \v_r $ is the error that resides in the outcome $ \xh $, containing both noise residuals and signal errors. Assuming that the denoising algorithm is effective, and $ \xh $ has an improved SNR compared to $ \y $, this means that 
\begin{align}
\frac{\| \x \|}{\| \v_r\|} \gg \frac{\| \x \|}{\| \v \|}
\end{align}
implying
\begin{align}
\label{star}
\| \v_r \| = \delta \| \v \|, \text{   where  } \delta \ll 1.
\end{align}
Thus, referring now to the addition $ \y + \xh $, its SNR satisfies
\begin{align}
\label{SNR2}
\text{SNR}^2(\y+\xh) &= \frac{\|2\x\|^2}{\|\v + \v_r\|^2} \\ \notag 
                       &\geq \frac{4\|\x\|^2}{\|\v\|^2 +2\| \v \| \|\v_r\| + \| \v_r\|^2}  \notag
\end{align}
In the above we used the Cauchy-Shwartz inequality. Using (\ref{star}) we get
\begin{align}
\label{SNR3}
\text{SNR}^2(\y+\xh)  & \geq \frac{4\|\x\|^2}{(1 + \delta)^2\|\v\|^2} \\ \notag
                      & = \frac{4}{(1 + \delta)^2}\text{SNR}^2(\y). \notag
\end{align}
Since $ \delta \ll 1 $, we have that 
\begin{align}
\label{SNR4}
\text{SNR}(\y+\xh) > \text{SNR}(\y),
\end{align}
where in the ideal case ($ \delta=0 $), the relation becomes
\begin{align}
\label{SNR5}
\text{SNR}(\y+\xh) = 2 \cdot \text{SNR}(\y).
\end{align}

\subsection{Convergence Analysis}
\label{convergence}
Studying the convergence of the SOS boosting is done by leveraging the linear matrix formulation of the denoising algorithm. The error of the SOS recursive function
\begin{align}
\label{err}
e_{k} = \xh^{k} - \xh^{*},
\end{align}
is defined as the difference between the $ k^{th} $ estimate,
\begin{align}
\label{x_k}
\xh^{k} = \W_{k}\left(  \y + \xh^{k-1}\right) - \xh^{k-1},
\end{align}
and the outcome that is obtained after a large number iterations, 
\begin{align}
\label{x_*}
\xh^{*} = \W_*\left(  \y + \xh^{*}\right) - \xh^{*},
\end{align}
where $ \W_k $ is a filter matrix, which is equivalent to applying $ f\left( \cdot \right) $ on the signal-strengthened image. Substituting Equations (\ref{x_k}) and (\ref{x_*}) into Equation (\ref{err}) lead to
\begin{align}
\label{err_k}
e_{k}  = & \W_{k}\left(  \y + \xh^{k-1}\right) - \xh^{k-1} -  \left( \W_*\left(  \y + \xh^{*}\right) - \xh^{*}\right)  \\ \notag
       = & \left( \W_k - \W_* \right)\y + \W_k\xh^{k-1} - \W_k\xh^{*} + \W_k\xh^{*} - \W_*\xh^{*}  - \left(\xh^{k-1} - \xh^{*}\right) \\ \notag
       = & \left( \W_k - \W_* \right)\y + \W_ke_{k-1}  + \left( \W_k - \W_*\right) \xh^{*}  - e_{k-1} \\ \notag      
       = &\left( \W_k-\I \right) e_{k-1} + \left( \W_k - \W_* \right)\left( \y + \xh^{*}\right),  \notag
\end{align}
where we use the recursive connection $ e_{k-1} = \xh^{k-1} - \xh^{*} $. We should note that the non-linearity of $ f(\y + \xh^{k-1}) = \W_k(\y + \xh^{k-1}) $ is neglected in the above derivation by allowing an operation of the form $ \W_k(\y + \xh^{k-1}) = \W_k\y + \W_k\xh^{k-1} $.

In the following convergence analysis we shall assume a fixed filter-matrix $ \W $ that operates on the signal-strengthened image along the whole SOS-steps, i.e., $\W = \W_k = \W_* $. This comes up in practice after applying the SOS boosting for a large number of iterations (as explained in the context of Figure \ref{plots}).  In this case, the above-mentioned abuse of the non-linearity becomes correct, and thus the convergence analysis is valid.
\begin{theorem}
\label{convergence_study}
Assume that $ \W = \W_k = \W_*$, and that the spectral radius of the transition matrix $ \|\W-\I\|_2 = \gamma <1 $. The error  $ e_k $ converges exponentially, i.e., $ \|e_{k}\|_2 \leq \|e_0\|_2 \cdot \gamma^k \to 0$ for $ k \to \infty $. Thus, the SOS recursive function is guaranteed to converge.
\end{theorem}
\begin{proof} By assigning $ \W_k = \W_* $, the second term $ \left( \W_k - \W_* \right)\left( \y + \xh^{*}\right) $ in Equation (\ref{err_k}) vanishes, thus
\begin{align}
\label{err_*}
e_{k}   & = \left( \W-\I \right) e_{k-1} \\ \notag
		& = \left( \W-\I \right)^k  e_{0} \notag
\end{align}
where $ e_{0} = \xh^{0} - \xh^{*} = - \xh^{*} $ is a constant vector. Using matrix-norm inequalities we get
\begin{align}
\label{err_norm}
\|e_{k}\|_2   & \leq  \| \W-\I \|_2^k \cdot \|  e_{0} \|_2 \\ \notag
              & = \gamma^{k} \cdot \|  e_{0} \|_2, \notag
\end{align}
where we use $ \| \W-\I \|_2 = \gamma $. As a result, $ \| e_{k} \|_2 $ is bounded by $ \gamma^{k} \cdot \|  e_{0} \|_2 $ and approaches zero for \mbox{$ k \to \infty $} when $ \gamma < 1 $.
\end{proof}

As such, the SOS boosting is guaranteed to converge for a wide range of denoising algorithms -- the ones that can be formulated/approximated such that $ \W-\I $ is convergent, e.g., the K-SVD \cite{KSVD_REF1}, NLM \cite{NL_DENOISE_REF4}, Bilateral filter \cite{tomasi1998bilateral} and LARK \cite{chatterjee2012patch}. In the next sub-section we soften the convergence requirements and intensify its properties, along with a practical demonstration.

\subsection{Parametrization}
\label{parametrization}
We generalize the SOS boosting algorithm by introducing two parameters that modify the steady-state outcome, the requirements for convergence (the eigenvalues range) and its rate. Starting with the first parameter, $ \rho $, which controls the signal emphasis, the formulation proposed is:
\begin{align}
\label{rho_sos}
\xh^{k+1} = f\left( \y + \rho\xh^{k}\right) - \rho\xh^k,
\end{align}
where a large value of $ \rho $ implies a strong emphasis of the underlying signal. Assigning $ \xh^{k+1} = \xh^{k} = \xh^{*}$ and replacing $ f\left( \cdot \right)  $ with a fixed filter-matrix $ \W_{*} $ lead to
\begin{align}
\xh^{*} = \W_*\left( \y + \rho\xh^{*}\right) - \rho\xh^{*},
\end{align}
which implies a steady-state result
\begin{align}
\label{closed_form_rho}
\xh^{*} = \left( \I + \rho(\I - \W_*) \right)^{-1}\W_*\y.
\end{align}
This is the new steady-state outcome, obtained only if the SOS boosting converges. We should note that this outcome also minimizes a cost function that involves the graph Laplacian as a regularizer (see Section \ref{graph}for further details). The conditions for convergence are studied hereafter.

The second parameter, $ \tau $, modifies the eigenvalues of the error's transition matrix, thereby leading to a faster convergence and relaxing the requirement that only $ f(\cdot) $ with eigenvalues between 0 to 1 is guaranteed to converge. We introduce this parameter in such a way that it will not affect that steady-state outcome (at least as far as the linear approximation is concerned). We start with the steady-state relation
\begin{align}
\label{sos_steady1}
\xh^{*} = f\left( \y + \rho\xh^{*}\right) - \rho\xh^*.
\end{align}
We multiply both sides by $ \tau $ and add the term $ \xh^*-\xh^* $ to the RHS, 
\begin{align}
\label{sos_steady2}
\tau\xh^{*} = \tau f\left( \y + \rho\xh^{*}\right) - \tau \rho\xh^* + \xh^{*} - \xh^{*}.
\end{align}
Thus, the same $ \xh^* $ solving (\ref{sos_steady1}) will also solve (\ref{sos_steady2}) and thus the steady-state is not affected. 
Rearranging this equality leads to 
\begin{align}
\label{sos_steady3}
\xh^{*} = \tau f\left( \y + \rho\xh^{*}\right) - (\tau\rho + \tau -1 )\xh^*.
\end{align}
As a result, the proposed generalized SOS boosting is given by
\begin{align}
\label{full_sos}
\xh^{k+1} = \tau f\left( \y + \rho\xh^{k}\right) - (\tau\rho + \tau -1 )\xh^k.
\end{align}
It is important to note that although $ \tau $ does not affect $ \xh^{*} $ explicitly, it may modify the estimates $ \xh^{k} $ over the iterations.  Due to the adaptivity of $ f(\cdot) $ to its input, such modifications may eventually affect the steady-state outcome.

Studying the convergence of Equation (\ref{full_sos}) is done in the same way that was taken in Section \ref{convergence}. Starting with the error computation, expressed by
\begin{align}
\label{err_k_full}
e_{k}  = & \xh^{k} - \xh^{*} \\ \notag  
       = & \tau\W_{k}\left(  \y + \rho\xh^{k-1}\right) - (\tau\rho + \tau -1 )\xh^{k-1} - \left( \tau\W_*\left(  \y + \rho\xh^{*}\right) - (\tau\rho + \tau -1 )\xh^{*} \right) \\ \notag
       = & \tau\left( \W_{k}- \W_* \right)\y + \tau\rho\W_{k}\xh^{k-1} - \tau\rho\W_k\xh^{*} + \tau\rho\W_k\xh^{*} - \tau\rho\W_*\xh^{*}  - (\tau\rho + \tau -1 )e_{k-1} \\ \notag
       = & \left( \tau\rho\W_k-(\tau\rho +\tau-1)\I \right) e_{k-1}  + \tau\left( \W_k - \W_* \right)\left( \y + \rho\xh^{*}\right).  \notag
\end{align}
Next, following Theorem \ref{convergence_study}, and assuming $ \W = \W_k = \W_* $, we get that the condition for convergence is:
\begin{align}
\label{tau_rho_condition}
\forall i \hspace{1em}  \phi(\tau, \rho, \lambda_i) = | \tau\rho\lambda_i - (\tau\rho + \tau) + 1 | < 1,
\end{align}
where $ \{ \lambda_i \}_{i=1}^{N} $ and $\{ \phi(\tau,\rho,\lambda_i) \}_{i=1}^{N}$ are the eigenvalues of $ \W $ and the error's transition matrix, respectively.
In order to achieve the fastest convergence, we seek for the parameter $ \tau^{*} $ that minimizes
\begin{align}
\label{best_tau_rho}
\tau^{*} = \min_{\tau} \max_{1 \le i \le N} \phi(\tau, \rho, \lambda_i) \hspace{0.7em} \text{s.t.} \hspace{0.7em} \forall i  \hspace{0.5em} \phi(\tau, \rho, \lambda_i)<1.
\end{align}
Given $ \tau^{*} $, the rate-of-convergence is governed by
\begin{align}
\label{best_gamma}
\gamma^{*} = \max_{1 \le i \le N} \phi(\tau^{*}, \rho, \lambda_i).
\end{align}
Appendix \ref{best_params} provides the following closed-form solution for Equation (\ref{best_tau_rho}),
\begin{align}
\label{tau_star}
\tau^{*} =  \frac{2}{2(\rho+1)-\rho(\lambda_{min} + \lambda_{max})},
\end{align}
along with optimal convergence rate,
\begin{align}
\label{gamma_star}
\gamma^{*} =  \frac{\rho(\lambda_{max}-\lambda_{min})}{2(\rho+1)-\rho(\lambda_{min} + \lambda_{max})}.
\end{align}

\begin{figure}[tbp]
\centering
\vspace{-0.2cm}
\mbox{ \subfigure[$ \log_{10}(\|e_k\|_2) $]{\epsfig{figure=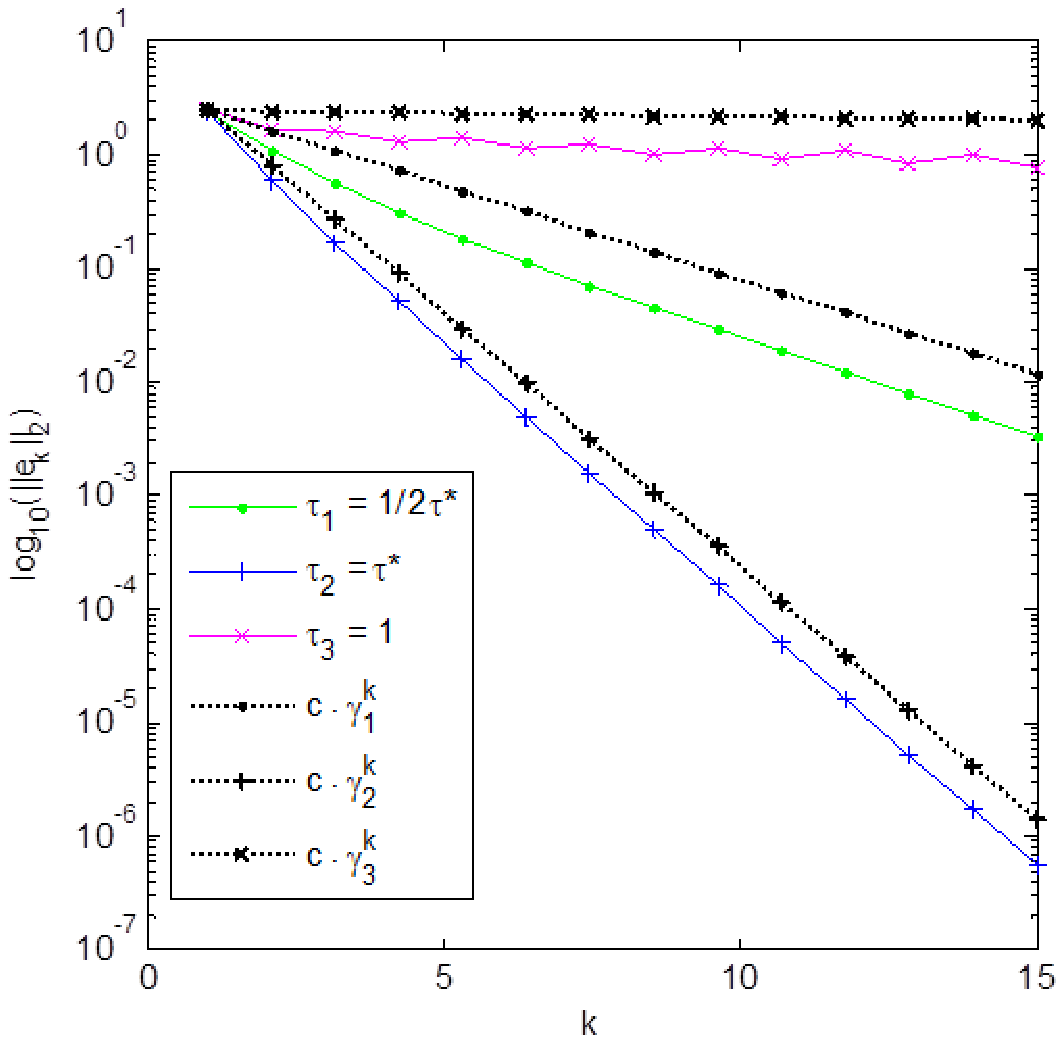, height = 1.9in} \label{plots::conv}}\quad 
 \subfigure[PSNR of $ \xh^k$]{\epsfig{figure = 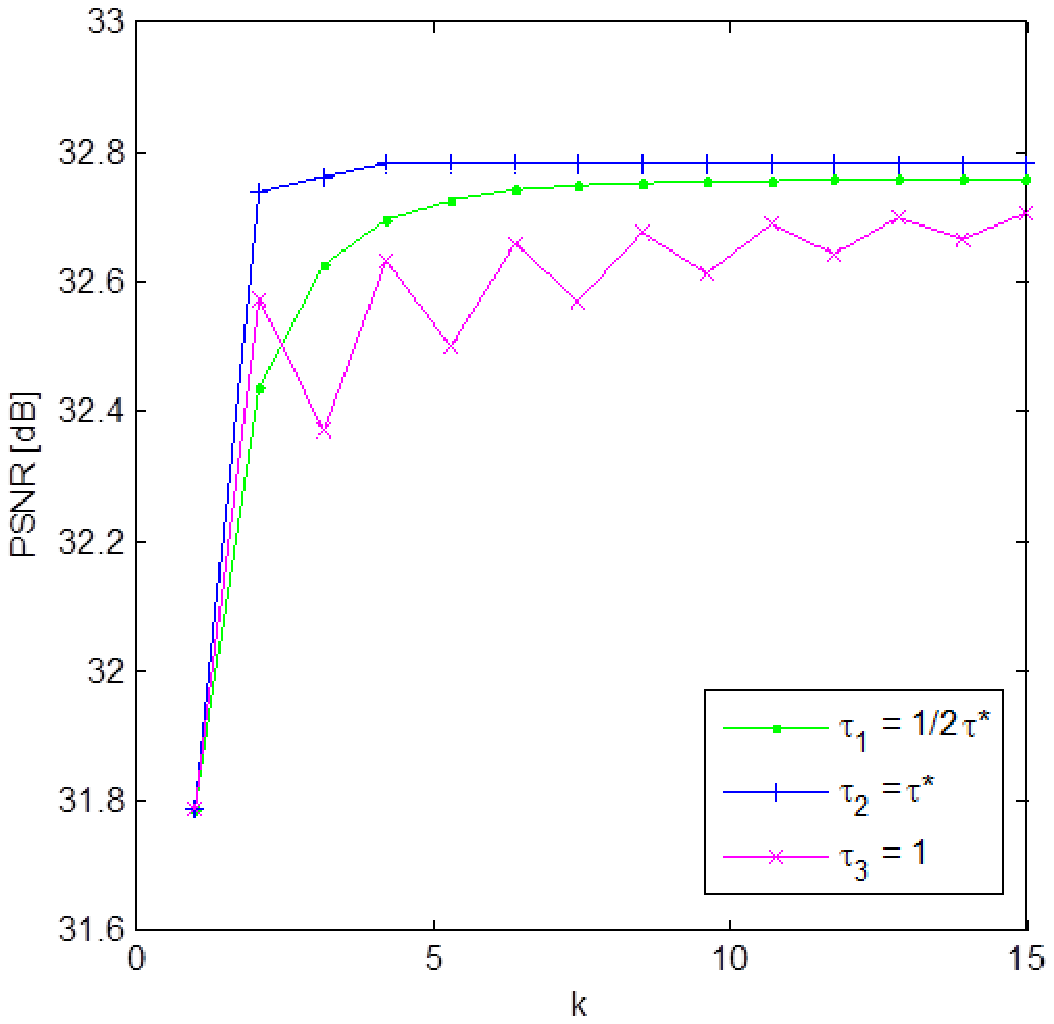, height = 1.9in} \label{plots::psnr}}}
 \vspace{-0.2cm}
\caption{Illustration of the generalized SOS recursive function properties: (a) convergence and (b) PSNR improvement. These graphs are generated by operating the K-SVD denoising \cite{KSVD_REF1} on noisy ($ \sigma = 25 $) \textsf{House} image.}
\label{plots}
\vspace{-0.5cm}
\end{figure}

In the context of the K-SVD image denoising \cite{KSVD_REF1}, Figure \ref{plots} demonstrates the properties of the generalized SOS recursive function for the image \textsf{House}, corrupted by zero-mean Gaussian noise with $ \sigma = 25 $. Each K-SVD operation includes 5 iterations of sparse-coding and dictionary-update with noise level of $ \hhsigma = 1.05\sigma $. We repeat these operations for 100 SOS-steps and set $ \xh^*=\xh^{100} $. In the following experiment, the denoised images are the outcome of $ \xh^{k+1} = \W(\y+\xh^{k}) - \xh^{k} $, where $ \W $ is held fixed as $ \W = \W_{100} $, initializing with $ \xh_0=\textbf{0} $. 

According to Theorem \ref{PropertiesofW} and based on the original K-SVD parameters ($ n = 8 $ , $m = 256$, $\mu = 1.02 $), we get $ \lambda_{min} \geq 0.015$ and $ \lambda_{max} = 1 $. These are leading to $ \tau^{*} = 0.67 $ (see Equation (\ref{tau_star})). Figure \ref{plots::conv} plots the logarithm of $ \|e_k\|_2 $ for $ [\tau_1, \tau_2, \tau_3] = [\frac{1}{2}\tau^{*}, \tau^{*}, 1] $. As can be seen, the error norm decreases linearly, and bounded by $ c\cdot \gamma_i^k $, where $[\gamma_1, \gamma_2, \gamma_3] = [0.66, 0.33, 0.98] $. The fastest convergence is obtained for $ \tau_2 = \tau^{*} = 0.67 $ with $ \gamma_2 =\gamma^{*}=0.33 $. While the slowest one is obtained for $ \tau_3=1 $ with $ \gamma_3 = 0.98 $. 

Figure \ref{plots::psnr} demonstrates the PSNR improvement (the higher the better) as a function of the SOS-step. As can be seen, faster convergence of $ \|e_k\|_2 $ translates well into faster improvement of the final image. The SOS boosting achieves PSNR of $ 32.78 $dB, offering an impressive improvement over the original K-SVD algorithm that obtains $ 31.8 $dB.

\section{Local-Global Interpretation}
\label{disagreement_algo}

As described in Section \ref{intro}, there is a stubborn gap between the local processing of image patches and the global need (creating a final image by aggregating the patches). Consider a denoising scenario based on overlapping patches (e.g. \cite{KSVD_REF1,PLE_REF}): At the local processing stage, each patch is denoised independently\footnote{Note that in our terminology, even methods like BM3D \cite{BM3D_REF} are considered as local, even though they share information between groups of patches. Indeed, our discussion covers this and related methods as well. In a way, the approach taken in \cite{ram2013image} offers some sort of remedy to the BM3D method.} without any influence from neighboring patches. Then, a global stage merges these outcomes by plainly averaging the local denoising results. 

Inspired by game-theory ideas, in particular the ''consensus and sharing'' optimization problem \cite{boyd2011distributed}, we introduce an interesting local-global interpretation to the above-proposed SOS boosting algorithm. A game theoretical terminology of a patch-based processing can be viewed as the following: There are several agents, where each one of them adjusts its local variable to minimize its individual cost (in our case -- representing the noisy patch sparsely). In addition, there is a shared objective term (the global image) that describes the overall goal. Imitating this concept, we name the following SOS interpretation as ''sharing the disagreement''. This approach, reported in \cite{romanosharing}, reduces the local-global gap by encouraging the overlapping patches to reach an agreement before they merge their forces by the averaging. 

The proposed boosting algorithm reduces the local-global gap in the following way. Per each patch, we define the difference between the local (intermediate) result and the patch from the global outcome as a ''disagreement''. Since each patch is denoised independently, such disagreement is almost always non-zero and even substantial. Sharing the information between the overlapping patches is done by subtracting the disagreement from the noisy image patches, i.e., seeking for an agreement between them. These modified patches are the new inputs to the denoising algorithm. In this way we push the overlapping patches to share their local results, influence each other and reduce the local-global gap.
 
More specifically, given an initial denoised version of $ \y $ and its intermediate patch results, we suggest repeating the following procedure: (i) compute the disagreement per each patch, (ii) subtract the result from the noisy input patches, (iii) apply the denoising algorithm to these modified patches, and (iv) reconstruct the image by averaging on the overlaps. Focusing on the K-SVD image denoising, this procedure is detailed in Algorithm \ref{consensus_algo}.
\vspace{-0.2cm}
\begin{algorithm}[H]
	\caption{\text{: Sharing the disagreement approach \cite{romanosharing}.}}
	\renewcommand{\algorithmicrequire}{\textbf{Initialization:}}
    \label{consensus_algo}
\begin{algorithmic}[1] 
\REQUIRE
\STATE	$ \D^0 \in \RR^{n \times m} $ -- initial dictionary.
\STATE	Set $ k=0 $.
\STATE 	Set $ \q_{i}^{0}=0 $, where $ \q_{i} \in \RR^n $ is a ''disagreement'' patch, corresponding to the $ i^{th} $ patch in the image.  
\end{algorithmic}
\renewcommand{\algorithmicrequire}{\textbf{Repeat}}
\begin{algorithmic}[1]
\REQUIRE
\STATE \textbf{Sparse-Coding and Dictionary Update:} Solve
\begin{align}
\label{con_min}
\left[ {\D^{k+1}}, \{\hat \alpha_i^{k+1}\}_{i=1}^{N} \right]  =
\min_{\D, \{\alpha_i\}_{i=1}^{N}}  \sum_{i=1}^{N}  \gamma_i\|\alpha_i\|_0 + \|\D\alpha_i - (\R_{i}\y - \q_{i}^k)\|_2^2.
\end{align}
In practice we approximate the representation of $ \R_{i}\y - \q_{i}^k $ using the OMP \cite{OMP_REF} and update the previous dictionary $ \D^k $ using the K-SVD algorithm \cite{KSVD_REF2}.
\STATE \textbf{Image Reconstruction:} Solve 
\begin{align}
\xh^{k+1} = \min_{z} \sum_{i} {\Arrowvert\D^{k+1} \alpha_{i}^{k+1}-\R_{i}z\Arrowvert_{2}^{2}}. 
\end{align}
This term leads to a simple averaging of the denoised patches $ \D^{k+1}\alpha_{i}^{k+1} $ on the overlaps.
\STATE \textbf{Disagreement Computation:} Update 
\begin{align}
\q_{i}^{k+1} = \D^{k+1}\alpha_{i}^{k+1} - \R_{i}\xh^{k+1},
\end{align} 
where $ \q_{i}^{k+1} $ is the ''disagreement'' between the independent estimation $ \D^{k+1}\alpha_{i}^{k+1} $ and the corresponding patch from the global outcome $ \R_{i}\xh^{k+1} $.
\end{algorithmic}
\renewcommand{\algorithmicrequire}{\textbf{Until}}
\begin{algorithmic}
\REQUIRE
\STATE Maximum denoising quality has been reached, else increment $ k $ and return to ''Sparse-Coding and Dictionary Update''. 
\end{algorithmic}
\renewcommand{\algorithmicrequire}{\textbf{Output}}
\begin{algorithmic}
\REQUIRE
\STATE $ \xh^{*} $ -- the last iteration result. 
\end{algorithmic}
\end{algorithm}

The modified input patches contain their neighbors information, thus encouraging the locally denoised patches to agree on the global result. Substituting $ \q_i^{k} =  \D^{k}\alpha_{i}^{k} - \R_{i}\xh^{k}$ in Equation (\ref{con_min}) leads to
\begin{align}
&\left[ {\D^{k+1}}, \{\hat \alpha_i^{k+1}\}_{i=1}^{N} \right]  =  \min_{\D, \{\alpha_i\}_{i=1}^{N}}  \sum_{i=1}^{N} \gamma_i\|\alpha_i\|_0 + \| \D\alpha_i - (\R_i\y - \D^{k}\alpha_{i}^{k} + \R_{i}\xh^{k})\|_2^2.
\end{align}
Now, by denoting the \emph{local} residual (method-noise) as \mbox{$ \r_i = \R_{i}\y - \D^{k}\alpha_{i}^{k} $}, we get
\begin{align}
\left[ {\D^{k+1}}, \{\hat \alpha_i^{k+1}\}_{i=1}^{N} \right]  = 
\min_{\D, \{\alpha_i\}_{i=1}^{N}} \sum_{i=1}^{N} \gamma_i\|\alpha_i\|_0 + \| \D\alpha_i - (\R_{i}\xh^{k} + \r_i)\|_2^2,
\end{align}
where the representation $ \D\alpha_i $ is the denoised version of the patch $ \R_{i}\xh^{k} + \r_i $. In this formulation, the input to the K-SVD is a patch from the global (previous iteration) cleaned image $ \R_{i}\xh^{k} $, contaminated by its own local method-noise $ \r_i $. Notice the major differences between Equation (\ref{twicing_algo}) that denoises the method-noise, Equation (\ref{bregman_algo}) that adds the method-noise to the noisy image and then denoises the result, and our \emph{local} approach that aims at recovering the previous global estimation, thereby leading to an agreement between the patches. Our algorithm is also different from the EPLL \cite{zoran2011learning}, which denoises the previous cleaned image without considering its method-noise. 

Still in the context of the K-SVD, Appendix \ref{sos_and_disagree} shows, under some assumptions, an \emph{equivalence} between the SOS recursive function (Equation (\ref{sos_algo})) and the above ''sharing the disagreement'' algorithm. It is important to emphasize that the former treats the K-SVD as a ''black-box'', thereby being blind to the K-SVD intermediate results (the independent denoised patches, before the patch averaging step). On the contrary, in the case of the disagreement approach, these intermediate results are crucial -- they are central in the algorithm. Therefore, the connection between the SOS and the disagreement algorithms is far from trivial.

\section{Graph Laplacian Interpretation}
\label{graph}
In this section we present a graph-based analysis to the SOS boosting. We start by providing a brief background on graph representation of an image in the context of denoising. Second, we explore the graph Laplacian regularization in general, and in the context of Equation (\ref{closed_form_rho}), the steady-state outcome of the SOS boosting. Finally, we suggest novel recursive algorithms (that treat the denoiser as a ''black-box'') to the graph Laplacian regularizers that are described in \cite{elmoataz2008nonlocal,bougleux2009local,PeymanLaplaceDenoising,PeymanLaplace}.    

Recent works \cite{gilboa2007nonlocal,gilboa2008nonlocal,elmoataz2008nonlocal,bougleux2009local,shuman2013emerging,liu2014progressive,symm,PeymanLaplaceDenoising,PeymanLaplace} suggest representing an image as a weighted graph $ \mathcal{G} = \left( {\boldsymbol{V}}, \boldsymbol{E}, \boldsymbol{K}\right) $, where the vertices $ \boldsymbol{V} $ represent the image pixels, the edges $ \boldsymbol{E} \subseteq \boldsymbol{V} \times \boldsymbol{V} $ represent the connection/similarity between pairs of pixels, with a corresponding weight $ \boldsymbol{K}(i,j) $. 

A constructive approach for composing a graph Laplacian for an image is via image denoising algorithms. Given a denoising process for an image, which can be represented as a matrix multiplication, $ {\xh} = \W \y $, 
one can refer to the entry $ (i,j) $ as revealing information about the proximity between the $ i $-th and $ j $-th pixels. We note that the existence of the matrix $ \W $ does not imply that the denoising process is linear. Rather, the non-linearity is hidden within the construction of the entries of $ \W $. For example, in the case of the NLM \cite{NL_DENOISE_REF4}, Bilateral \cite{tomasi1998bilateral} and LARK \cite{chatterjee2012patch} filters, the entries of $ \boldsymbol{K} $ can be expressed by
\begin{align}
\boldsymbol{K}(i,j) = \exp\left( -\frac{d^2(i,j)}{h^2}\right), 
\end{align}
where $ d(i,j) $ measures the distance between the $ (i,j) $ pixels (or patches), and $ h $ is a smoothing parameter. Notice that in the case of the sparsity-based K-SVD denoising \cite{KSVD_REF1}, the weights $ \boldsymbol{K}(i,j) $, as defined in Equation (\ref{denosied_image_filtered}), measure the similarity between the $ (i,j) $ pixels through the dictionary $ \D $. Dealing with an undirected graph $ \mathcal{G} $, the degree $ d_i $ of the vertex $ \boldsymbol{V}_i $ can be defined by 
\begin{align}
\label{degree}
\boldsymbol{D}(i,i) = d_i  = \sum_{j}\boldsymbol{K}(i,j),
\end{align}
where $ d_i $ is a sum over the weights on the edges that are connected to $ \boldsymbol{V}_i $, and $ \boldsymbol{D} $ is a diagonal matrix (called the degree matrix), containing the values of $ \{ d_i \}_{i=1}^{N}$ in its diagonal\footnote{The K-SVD degree matrix, as defined in Equation (\ref{denosied_image_filtered}), also holds the relation described in Equation (\ref{degree}). According to Theorem \ref{PropertiesofW}, $ \one $ is an eigenvector of $ \W=\boldsymbol{D}^{-1}\boldsymbol{K} $, corresponding to eigenvalue $ \lambda=1 $, leading to $ \boldsymbol{D}^{-1}\boldsymbol{K}\one = \one $. Multiplying both sides by $ \boldsymbol{D} $ results in the desired relation $ \boldsymbol{K}\one = \boldsymbol{D}\one $, i.e., $ \boldsymbol{D}(i,i)  = \sum_{j}\boldsymbol{K}(i,j) $.}.

The graph Laplacian has a major importance in describing functions on a graph \cite{von2007tutorial}, and in the case of image denoising -- representing the structure of the underlying signal \cite{elmoataz2008nonlocal,bougleux2009local,liu2014progressive,symm,PeymanLaplace}. There are several definitions of the graph Laplacian. In the context of the proposed SOS boosting, we shall use a normalized Laplacian, defined as
\begin{align}
\mathcal{L} = \I-\W,
\end{align}
where $ \W $ is a filter matrix, representing the denoiser $ f(\cdot) $ (see Equation (\ref{denosied_image_filtered})). Note that $ \W $ is a normalized version of the similarity matrix $ \boldsymbol{K} $, thus has eigenvalues in a range of 0 to 1. There are several ways to obtain $ \W $ from $ \boldsymbol{K} $, e.g., $ \W=\boldsymbol{D}^{-1}\boldsymbol{K} $ is used in \cite{szlam2008regularization} and in this work (leading to a random walk Laplacian), another way is $ \W=\boldsymbol{D}^{-1/2}\boldsymbol{K}\boldsymbol{D}^{-1/2} $ as used in \cite{meyer2014perturbation}. Recently, Kheradmand and Milanfar \cite{PeymanLaplace} suggest $ \W=\boldsymbol{C}^{-1/2}\boldsymbol{K}\boldsymbol{C}^{-1/2} $, where $ \boldsymbol{C} $ is the outcome of Sinkhorn algorithm \cite{knight2012fast}. Notice that different versions of $ \W $ result in different properties of $ \mathcal{L} $ (refer to \cite{PeymanLaplace} for more information).

In general, the spectrum of a graph is defined by the eigenvectors and eigenvalues of $ \mathcal{L} $. In the context of image denoising, as argued in \cite{gadde2013bilateral,meyer2014perturbation,symm,PeymanLaplace}, the eigenvectors that correspond to the small eigenvalues of $ \mathcal{L} $ encapsulate the underlying structure of the image. On the other hand, the eigenvectors that correspond to the large eigenvalues mostly represent the noise. Meyer et al. \cite{meyer2014perturbation} showed that the small eigenvalues are stable even for high noise scenarios. As a result, the graph Laplacian can be used as a regularizer, preserving the geometry of the image by encouraging similar pixels to remain similar in the final estimate \cite{elmoataz2008nonlocal,bougleux2009local}. 

What can we do with $ \mathcal{L} $? The obvious usage of it is as an image adaptive regularizer in inverse problems. There are several ways to integrate $ \mathcal{L} $ in a cost function, for example, \cite{elmoataz2008nonlocal,bougleux2009local} suggest solving the following minimization problem\footnote{The work in \cite{symm} is closely related, but their regularization term is $\|\mathbf{L} \x\|_2^2$, and thus it leads to $ \mathbf{L}^T\mathbf{L} $ in the steady-state formula, where $ \mathbf{L} = \boldsymbol{D} - \boldsymbol{K}$ is an un-normalized graph Laplacian. Thus, we omit it from the next discussion.} 
\begin{align}
\label{laplacian_reg}
\xh  = \min_{\x} {\Arrowvert \x-\y \Arrowvert}_2^2  + \rho\x^T\mathcal{L}\x,
\end{align}
leading to a closed-form expression for $ \xh $,
\begin{align}
\label{laplacian_reg_closed}
\xh  = (\I + \rho\mathcal{L})^{-1}\y.
\end{align}
The authors of \cite{PeymanLaplace} suggest an iterative graph-based framework for image restoration. Specifically, in the case of image denoising \cite{PeymanLaplaceDenoising}, they suggest a variant to Equation (\ref{laplacian_reg}),
\begin{align}
\label{laplacian_reg_peyman}
\xh  = \min_{\x} { (\x-\y)^T\W(\x-\y)  + \rho\x^T\mathcal{L}\x}.
\end{align}
Differently from Equation (\ref{laplacian_reg}), the above expression offers a weighted data fidelity term, resulting in the following closed-form expression to the final estimate:
\begin{align}
\label{laplacian_reg_closed_peyman}
\xh  = (\W + \rho\mathcal{L})^{-1}\W\y.
\end{align}
It turns out that Equation (\ref{closed_form_rho}), the steady-state result of the SOS boosting, i.e., 
\begin{align}
\label{closed_form_rho_lap}
\xh^* & = \left( \I + \rho(\I - \W_*) \right)^{-1}\W_*\y \\ \notag
    & = \left( \I + \rho \mathcal{L}_*\right)^{-1}\W_*\y,
\end{align}
can be also treated as emerging from a graph Laplacian regularizer, being the outcome of the following cost function
\begin{align}
\label{laplacian_reg_SOS}
\xh^*  = \min_{\x} {\Arrowvert \x-\W_*\y \Arrowvert}_2^2  + \rho\x^T\mathcal{L}_*\x.
\end{align}
Notice the differences between Equations (\ref{laplacian_reg}), (\ref{laplacian_reg_peyman}), and (\ref{laplacian_reg_SOS}). The last expression suggests that SOS aims to find an image that is close to the estimated image $ \W_*\y $, rather than the noisy $ \y $  itself. In the spirit of the SOS boosting,
\begin{align}
\xh^{k+1} = f\left( \y + \rho\xh^{k}\right) - \rho\xh^k, \notag
\end{align}
we can suggest expressing the above-mentioned graph Laplacian regularization methods, i.e., Equations (\ref{laplacian_reg_closed}) and (\ref{laplacian_reg_closed_peyman}), as recursive, providing novel ''black-box'' iterative algorithms that minimize their corresponding penalty functions without explicitly building the matrix $ \W $. Starting with Equation (\ref{laplacian_reg_closed}), the steady-state outcome should satisfy
\begin{align}
\label{steady_state_lap}
\left( \I + \rho(\I - \W) \right)\xh = \y.
\end{align}
There are many ways to rearrange this expression using the fixed point strategy, in order to get a recursive update formula. We shall adopt a path that leads to an iterative process that operates on the strengthened image, $ \y+\xh^k $, in order to expose the similarities and differences to our scheme. Therefore, we suggest adding $ \W\y - \W\y $ to the RHS, i.e.,
\begin{align}
\label{steady_state_lap_add}
\xh + \rho\xh - \rho\W\xh = \y + \W\y - \W\y.
\end{align}
Rearranging the above expression results in
\begin{align}
\label{steady_state_lap_add_rearange}
\xh = \frac{1}{(1+\rho)}\left[ \W \left(\y + \rho\xh \right) + \left( \y - \W\y \right)\right] .
\end{align}
As a consequence, the obtained iterative ''black-box'' formulation to the conventional graph Laplacian regularization \cite{elmoataz2008nonlocal,bougleux2009local} is given by
\begin{align}
\label{lap_iterative_func}
\xh^{k+1} = \frac{1}{(1+\rho)}\left[ f \left(\y + \rho\xh^{k} \right) + \left( \y - f\left(\y\right) \right)\right].
\end{align}
As can be seen, we got an iterative algorithm that, similar to SOS, operates on the strengthened image. However, rather than simply subtracting $ \rho\xh^k $ from the outcome, we add the method noise, and then normalize.

In a similar way, Equation (\ref{laplacian_reg_closed_peyman}), which is formulated as
\begin{align}
\label{steady_state_lap_Peyman}
\left( \W + \rho(\I - \W) \right)\xh = \W\y,
\end{align}
can be expressed by
\begin{align}
\label{steady_state_lap_Peyman_rearange}
\xh = \frac{1}{\rho}\W(\rho\xh + \y -\xh),
\end{align}
and in the general case, the ''black-box'' version of \cite{PeymanLaplaceDenoising} is formulated by
\begin{align}
\label{steady_state_lap_Peyman_func}
\xh^{k+1} = \frac{1}{\rho}f(\rho\xh^{k} + \y - \xh^{k}) =\frac{1}{\rho'+1} f(y+\rho' \xh^k),
\end{align}
where $ \rho'=\rho-1 $. Again, we see a close resemblance to our SOS method. However, instead of subtracting $ \rho\xh^k $ from the denoised strengthened image, we simply normalize accordingly. 

Equations (\ref{lap_iterative_func}) and (\ref{steady_state_lap_Peyman_func}) offer two iterative algorithms that are essentially minimizing the penalty functions (\ref{laplacian_reg}) and (\ref{laplacian_reg_peyman}), respectively. However, these algorithms offer far more -- both can be applied with the denoiser as a ''black-box'', implying that no explicit matrix construction of $ \W $ (nor $ \mathcal{L} $) is required. Furthermore, these schemes, given in the form of denoising on the strengthened image, imply that  parameter setting is trivial -- the noise level is nearly $ \sigma $, regardless of the iteration number. Lastly, an update of $ \W $ within the iterations of these recursive formulas seems most natural.

\section{Experimental Results}
\label{experiments}

\begin{figure}[hbp]
\centering
\mbox{\subfigure[\textsf{Foreman}]{\epsfig{figure=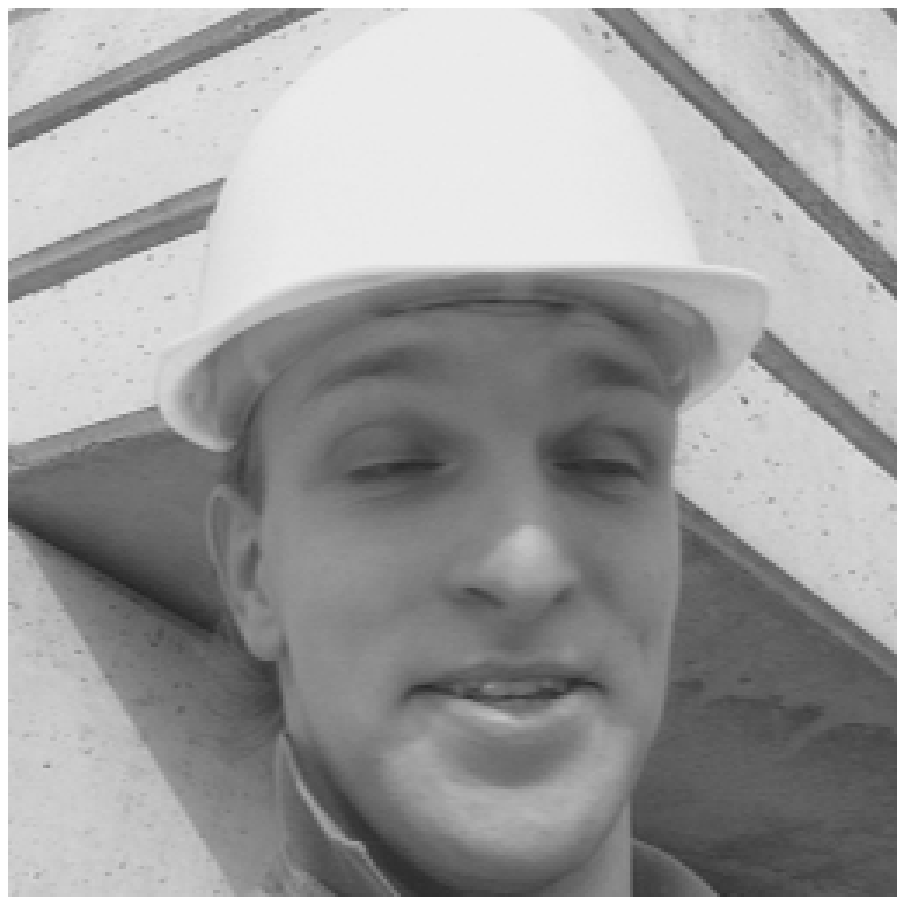, width = 1.175in}} 
 \subfigure[\textsf{Lena}]{\epsfig{figure = 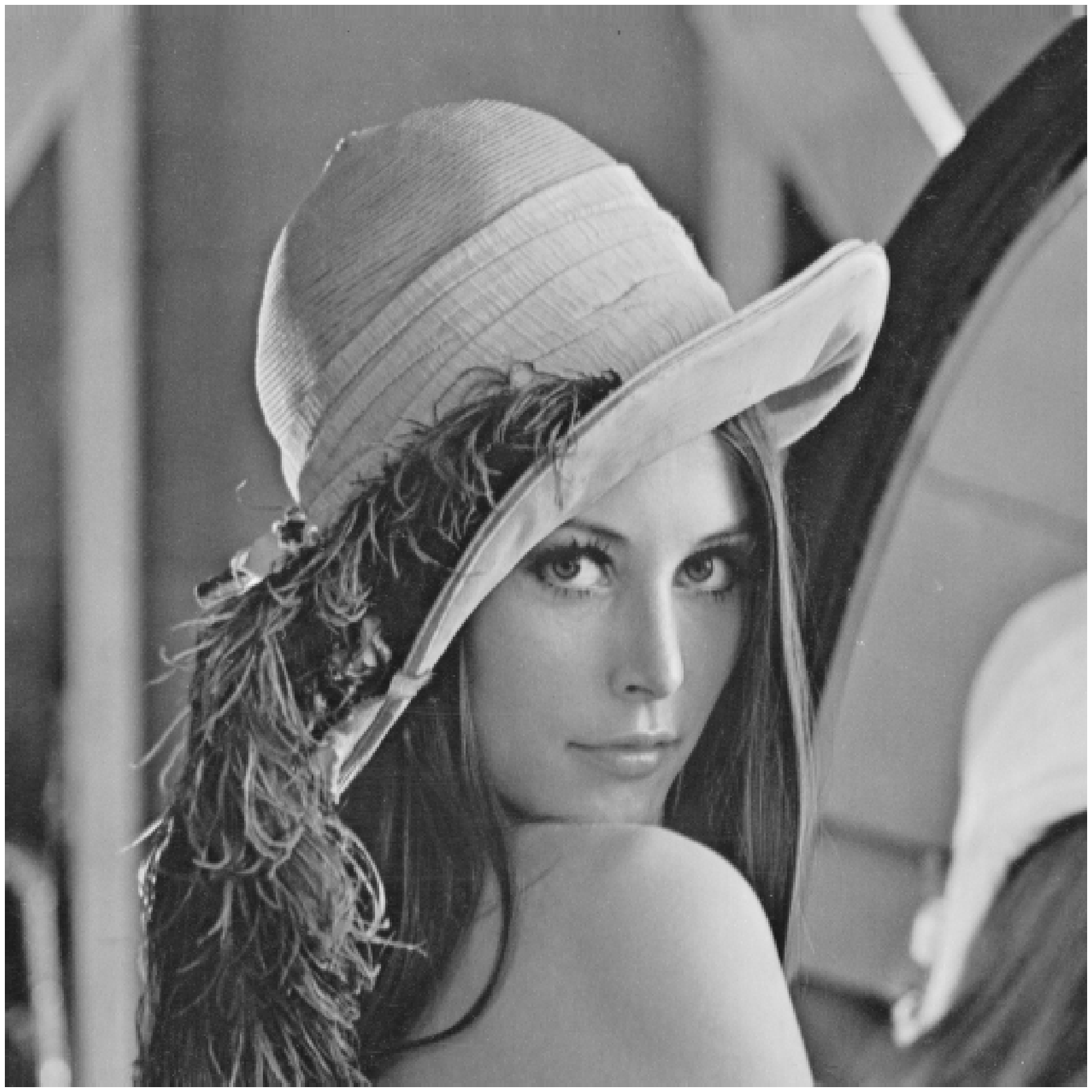, width = 1.175in}}
 \subfigure[\textsf{House}]{\epsfig{figure=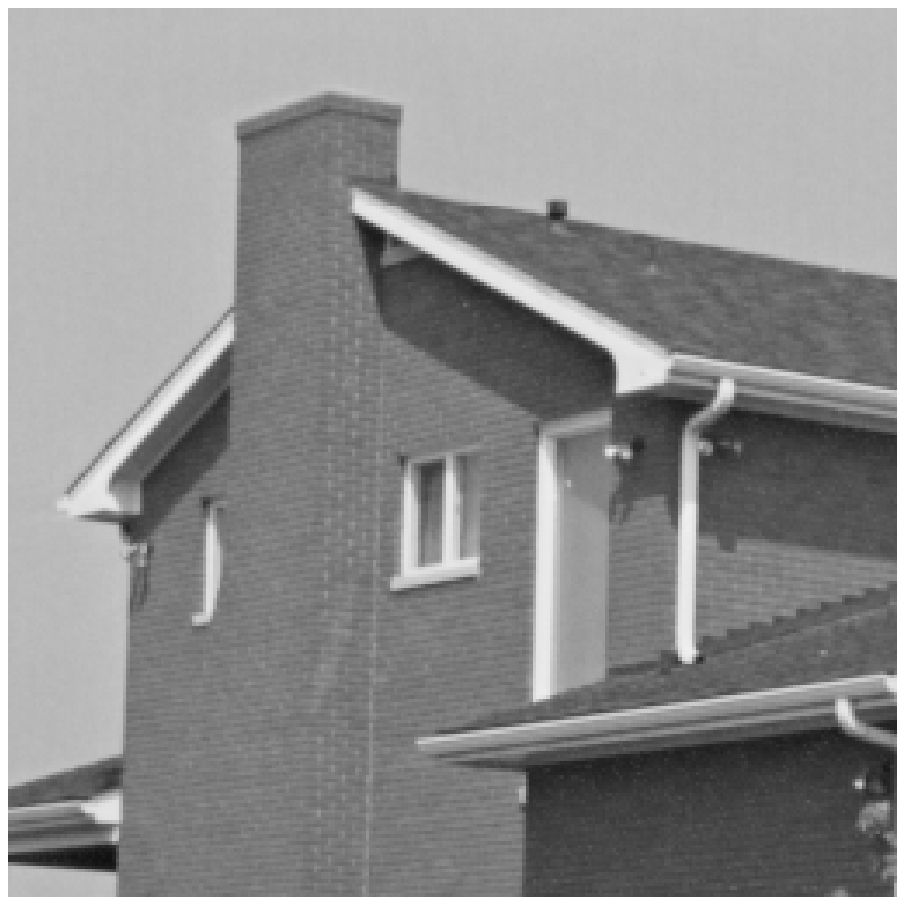, width = 1.175in}} 
 \subfigure[\textsf{Fingerprint}]{\epsfig{figure=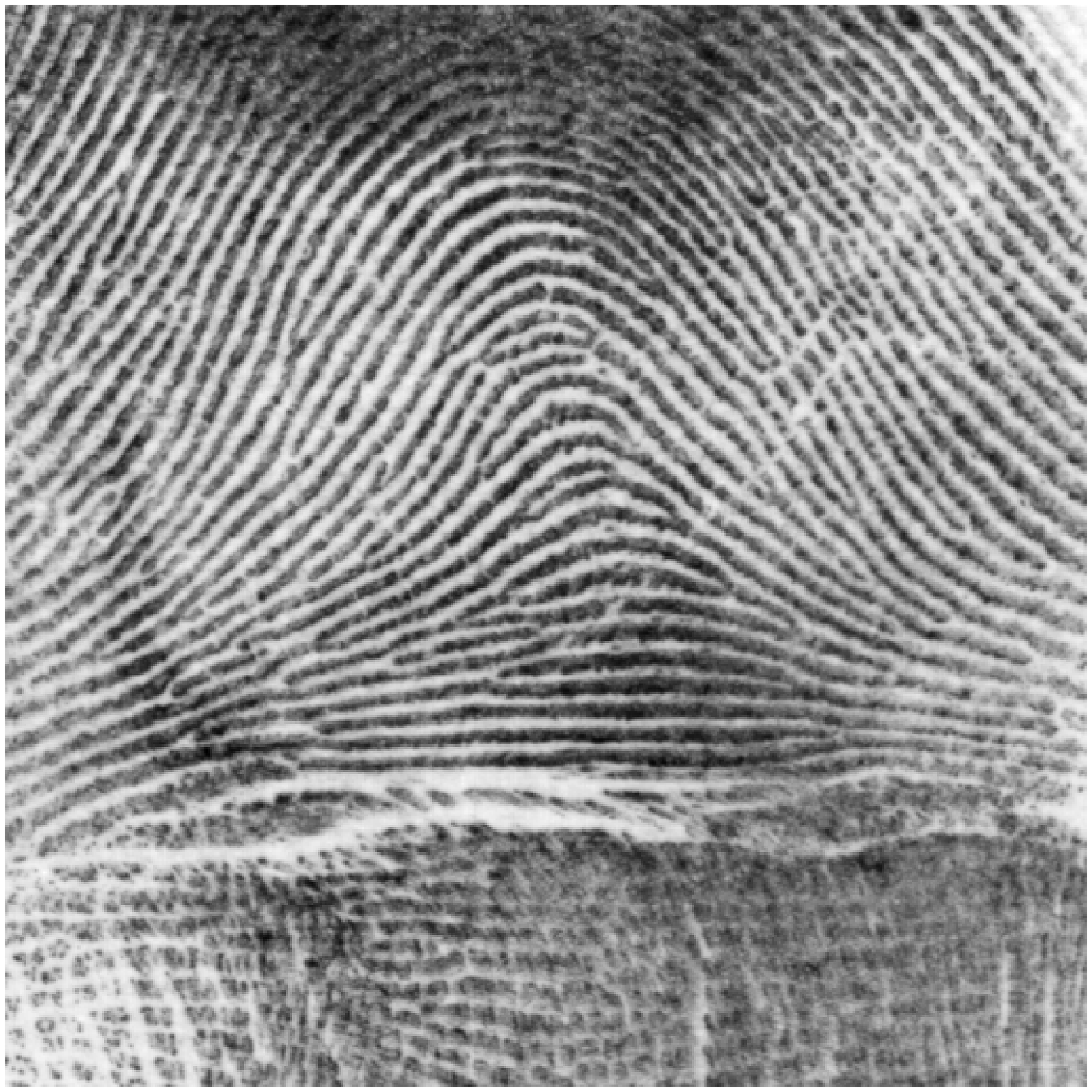, width = 1.175in}} 
 \subfigure[\textsf{Peppers}]{\epsfig{figure=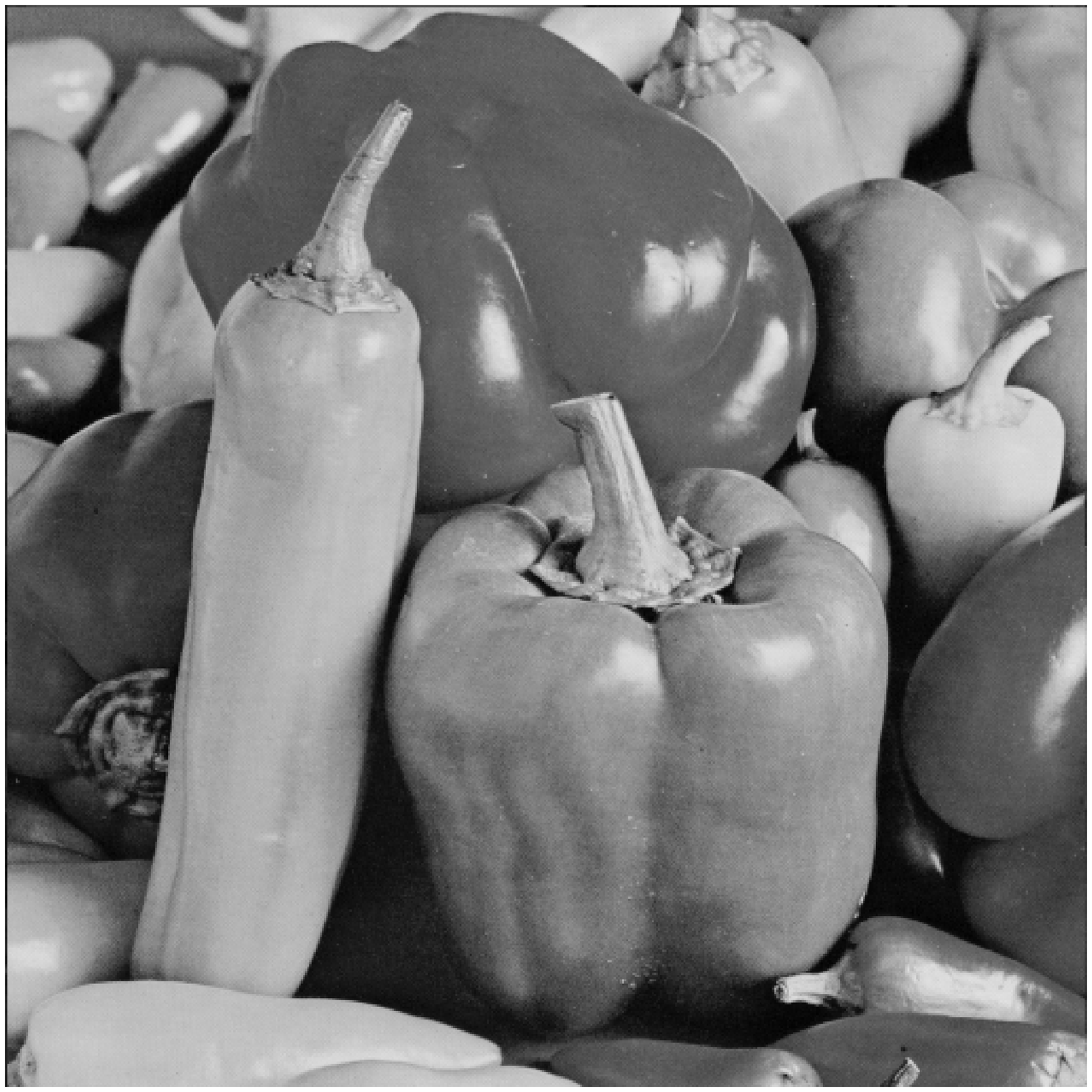, width = 1.175in}}}
\caption{Visualization of the test images.}
\label{testImages}
\end{figure}

In this section, we provide detailed results of the SOS boosting and its local-global variant -- ''sharing the disagreement''. The results are presented for the images \textsf{Foreman}, \textsf{Lena}, \textsf{House}, \textsf{Fingerprint} and \textsf{Peppers} (see Figure \ref{testImages}). These images are extensively tested in earlier work, thus enabling a convenient and fair demonstration of the potential of the proposed boosting.
The images are corrupted by an additive zero-mean Gaussian noise with a standard-deviation $ \sigma $. The denoising performance is evaluated using the Peak Signal to Noise Ratio (PSNR), defined as $ 20\log_{10}(\frac{255}{\sqrt{\text{MSE}}})$, where $\text{MSE}$ is the Mean Squared Error between the original image and its denoised version.

\subsection{SOS Boosting with state-of-the-art algorithms}

The proposed SOS boosting is applicable to a wide range of denoising algorithms. We demonstrate its abilities by improving several state-of-the-art methods: (i) K-SVD \cite{KSVD_REF1}, (ii) NLM \cite{NL_DENOISE_REF4, ipol_nlm}, (iii) BM3D \cite{BM3D_REF}, and (iv) EPLL \cite{zoran2011learning}. The K-SVD \cite{KSVD_REF1}, which was discussed in detail in this paper, is based on an adaptive sparsity model. The NLM \cite{NL_DENOISE_REF4} leverages the ''self-similarity'' property of natural images, i.e., the assumption that each patch may have similar patches within the image. The BM3D \cite{BM3D_REF} combines the ''self-similarity'' property with a sparsity model, achieving the best restoration and even touches some recently developed image denoising bounds \cite{levin2011natural}. The EPLL \cite{zoran2011learning}, which was described in Section \ref{intro}, represents the image patches using the Gaussian Mixture Model (GMM), and encourages the global result to comply with the local patches prior. As can be inferred, these algorithms are diverse and build upon different models and forces. Furthermore, the EPLL can be considered as a boosting method by-itself, designed to improve a GMM denoising algorithm. The diversity of the above algorithms emphasizes the potential of the SOS boosting.

The improved denoising performance is gained simply by applying the authors' original software as a ''black-box'', \emph{without any internal algorithmic modifications or parameters settings}\footnote{The original K-SVD uses $ 8 \times 8 $ patches, but our experiments show that $ 9 \times 9 $ yields nearly the same results for the core algorithm, while enabling better improvement with the SOS boosting. As a consequence, in the following experiments we demonstrate the results of the $ 9 \times 9 $ version.}. Such modifications may lead to better results and we leave these for future study. In order to apply SOS boosting we need to set the parameters $ \rho $, $ \tau $, and a modified noise-level $ \hhsigma $ (although $ \sigma $ is known). The parameter $ \hhsigma $, which might be a little higher than $ \sigma $, represents the noise-level of $ \y+\rho\xh^k $. We can estimate $ \hhsigma $ automatically (e.g using \cite{zoran2009scale}) or tunning a fixed value manually. In the following experiments we choose the second option. We set $ \tau = 1 $ (the effect of $ \tau^* $ is demonstrated later on) and run several tests to tune $ \rho $ and $ \hhsigma $ per each noise level and denoising algorithm, as detailed in Table \ref{tab:sos} under the 'SOS params' column.

In the case of the EPLL and BM3D, the authors' software is designed to denoise an input image in the range of 0 to 1. As such, we apply the SOS boosting ($ \tau=1 $) in the following formulation:
\begin{align}
\xh^{k+1} = \frac{1}{1-\tilde{\rho}} \cdot f\left( (1-\tilde{\rho})y+\tilde{\rho}\xh^k\right)  - \frac{\tilde{\rho}}{1-\tilde{\rho}} \cdot \xh^k ,
\end{align}
with a corresponding $ \tilde{\sigma} $. In order to remain consistent with the SOS parameters of the K-SVD and NLM, which apply Equation (\ref{rho_sos}), we provide hereafter the parameters $ \rho =  \frac{\tilde{\rho}}{1-\tilde{\rho}}$ and $ \hhsigma = \frac{\tilde{\sigma}}{1-\tilde{\rho}} $ for the EPLL and BM3D. 

\begin{table}[!htbf]		
\caption{Comparison between the denoising results [PSNR] of various algorithms (K-SVD \cite{KSVD_REF1}, NLM \cite{ipol_nlm}, BM3D \cite{BM3D_REF} and EPLL \cite{zoran2011learning}) and their SOS boosting outcomes. Per each denoising algorithm, we apply the authors' original software with the SOS formulation (using $\tau=1$, with the appropriate $ \rho $ and $ \hhsigma $). The best results per each denoising algorithm, image, and noise level are highlighted.}																																		
\begin{center}\footnotesize
\renewcommand{\arraystretch}{1.2}
\renewcommand{\tabcolsep}{0.022cm}
\begin{tabular}{|c||c|c||c|c||c|c||c|c||c|c||c|c||c|c|c|} \hline																																		
\multicolumn{16}{|c|}{K-SVD \cite{KSVD_REF1}}  \\  \cline{1-16}																																		
\multirow{2}[0]{*}{$ \sigma $} & \multicolumn{2}{c||}{{SOS params}} & \multicolumn{2}{c||}{\textsf{Foreman}} & \multicolumn{2}{c||}{\textsf{Lena}} & \multicolumn{2}{c||}{\textsf{House}} & \multicolumn{2}{c||}{\textsf{Fingerprint}} & \multicolumn{2}{c||}{\textsf{Peppers}} &  \multicolumn{3}{c|}{Average} \\ \cline{2-16}																																		
& $ \rho $ & $ \hhsigma $& Orig & SOS      & Orig & SOS      & Orig & SOS      & Orig & SOS      & Orig & SOS & Orig & SOS & Imprv.\\ \hline																																		
{10	}&{	0.30	}&{ $	1.00	\sigma$	}&{	36.92	}& \textbf{	37.13	}&{	35.47	}& \textbf{	35.58	}&{	36.25	}& \textbf{	36.49	}&{	32.27	}& \textbf{	32.35	}&{	34.68	}& \textbf{	34.71	}&{	35.12	}& \textbf{	35.25	}&	0.13	\\ \hline		
{20	}&{	0.60	}&{ $	1.00	\sigma$	}&{	33.81	}& \textbf{	34.11	}&{	32.43	}& \textbf{	32.67	}&{	33.34	}& \textbf{	33.62	}&{	28.31	}& \textbf{	28.54	}&{	32.29	}& \textbf{	32.35	}&{	32.04	}& \textbf{	32.26	}&	0.22	\\ \hline		
{25	}&{	1.00	}&{ $	1.00	\sigma$	}&{	32.83	}& \textbf{	33.12	}&{	31.32	}& \textbf{	31.62	}&{	32.39	}& \textbf{	32.72	}&{	27.13	}& \textbf{	27.44	}&{	31.43	}& \textbf{	31.49	}&{	31.02	}& \textbf{	31.28	}&	0.26	\\ \hline		
{50	}&{	1.00	}&{ $	1.00	\sigma$	}&{	28.88	}& \textbf{	29.85	}&{	27.75	}& \textbf{	28.37	}&{	28.01	}& \textbf{	28.98	}&{	23.20	}& \textbf{	23.98	}&{	28.16	}& \textbf{	28.66	}&{	27.20	}& \textbf{	27.97	}&	0.77	\\ \hline		
{75	}&{	1.00	}&{ $	1.00	\sigma$	}&{	26.24	}& \textbf{	27.32	}&{	25.74	}& \textbf{	26.40	}&{	25.23	}& \textbf{	26.85	}&{	19.93	}& \textbf{	21.88	}&{	25.73	}& \textbf{	26.72	}&{	24.57	}& \textbf{	25.83	}&	1.26	\\ \hline		
{100	}&{	1.00	}&{ $	1.00	\sigma$	}&{	25.21	}& \textbf{	25.39	}&{	24.50	}& \textbf{	24.99	}&{	23.69	}& \textbf{	24.59	}&{	17.98	}& \textbf{	19.61	}&{	24.17	}& \textbf{	25.03	}&{	23.11	}& \textbf{	23.92	}&	0.81	\\ \hline	\hline

																																		
\multicolumn{16}{|c|}{NLM \cite{ipol_nlm}}  \\  \hline																																		
\multirow{2}[0]{*}{$ \sigma $} & \multicolumn{2}{c||}{{SOS params}} & \multicolumn{2}{c||}{\textsf{Foreman}} & \multicolumn{2}{c||}{\textsf{Lena}} & \multicolumn{2}{c||}{\textsf{House}} & \multicolumn{2}{c||}{\textsf{Fingerprint}} & \multicolumn{2}{c||}{\textsf{Peppers}} &  \multicolumn{3}{c|}{Average} \\ \cline{2-16}																																		
& $ \rho $ & $ \hhsigma $& Orig & SOS      & Orig & SOS      & Orig & SOS      & Orig & SOS      & Orig & SOS & Orig & SOS & Imprv.\\ \hline																																		
{10	}&{	0.10	}&{ $	1.20	\sigma$	}&{	35.55	}& \textbf{	36.13	}&{	34.32	}& \textbf{	34.72	}&{	34.93	}& \textbf{	35.39	}&{	31.04	}& \textbf{	31.45	}&{	34.02	}& \textbf{	34.37	}&{	33.97	}& \textbf{	34.41	}&	0.44	\\ \hline		
{20	}&{	0.10	}&{ $	1.10	\sigma$	}&{	32.78	}& \textbf{	33.15	}&{	31.59	}& \textbf{	31.84	}&{	32.40	}& \textbf{	32.86	}&{	27.26	}& \textbf{	27.55	}&{	31.49	}& \textbf{	31.78	}&{	31.10	}& \textbf{	31.44	}&	0.34	\\ \hline		
{25	}&{	0.40	}&{ $	1.10	\sigma$	}&{	31.26	}& \textbf{	31.88	}&{	30.51	}& \textbf{	30.88	}&{	31.22	}& \textbf{	31.87	}&{	26.20	}& \textbf{	26.22	}&{	30.47	}& \textbf{	30.85	}&{	29.93	}& \textbf{	30.34	}&	0.41	\\ \hline		
{50	}&{	0.50	}&{ $	1.05	\sigma$	}&{	27.62	}& \textbf{	28.05	}&{	27.31	}& \textbf{	27.57	}&{	27.42	}& \textbf{	28.00	}&{	23.00	}& \textbf{	23.06	}&{	26.79	}& \textbf{	26.97	}&{	26.43	}& \textbf{	26.73	}&	0.30	\\ \hline		
{75	}&{	0.60	}&{ $	1.05	\sigma$	}&{	25.38	}& \textbf{	26.06	}&{	25.12	}& \textbf{	25.75	}&{	24.59	}& \textbf{	25.49	}&{	20.84	}& \textbf{	21.13	}&{	24.63	}& \textbf{	24.94	}&{	24.11	}& \textbf{	24.67	}&	0.56	\\ \hline		
{100	}&{	0.60	}&{ $	1.05	\sigma$	}&{	23.82	}& \textbf{	24.21	}&{	23.71	}& \textbf{	24.17	}&{	23.07	}& \textbf{	23.45	}&{	19.50	}& \textbf{	19.67	}&{	23.27	}& \textbf{	23.65	}&{	22.67	}& \textbf{	23.03	}&	0.36	\\ \hline	\hline

																																		
\multicolumn{16}{|c|}{BM3D \cite{BM3D_REF}}  \\  \cline{1-16}																																		
\multirow{2}[0]{*}{$ \sigma $} & \multicolumn{2}{c||}{{SOS params}} & \multicolumn{2}{c||}{\textsf{Foreman}} & \multicolumn{2}{c||}{\textsf{Lena}} & \multicolumn{2}{c||}{\textsf{House}} & \multicolumn{2}{c||}{\textsf{Fingerprint}} & \multicolumn{2}{c||}{\textsf{Peppers}} &  \multicolumn{3}{c|}{Average} \\ \cline{2-16}																																		
& $ \rho $ & $ \hhsigma $& Orig & SOS      & Orig & SOS      & Orig & SOS      & Orig & SOS      & Orig & SOS & Orig & SOS & Imprv.\\ \hline																																		
{10	}&{	0.05	}&{ $	1.02	\sigma$	}&{	37.23	}& \textbf{	37.24	}&{	35.84	}& \textbf{	35.85	}&{	36.54	}& \textbf{	36.55	}&{	32.46	}& \textbf{	32.47	}&{	34.96	}&{	34.96	}&{	35.40	}& \textbf{	35.41	}&	0.01	\\ \hline		
{20	}&{	0.11	}&{ $	1.03	\sigma$	}&{	34.50	}& \textbf{	34.55	}&{	33.00	}& \textbf{	33.02	}&{	33.81	}&{	33.81	}&{	28.82	}& \textbf{	28.83	}&{	32.67	}& \textbf{	32.68	}&{	32.56	}& \textbf{	32.58	}&	0.02	\\ \hline		
{25	}&{	0.18	}&{ $	1.04	\sigma$	}&{	33.41	}& \textbf{	33.48	}&{	32.02	}& \textbf{	32.04	}&{	32.90	}&{	32.90	}&{	27.72	}&{	27.72	}&{	31.87	}& \textbf{	31.89	}&{	31.58	}& \textbf{	31.61	}&	0.03	\\ \hline		
{50	}&{	0.25	}&{ $	1.04	\sigma$	}&{	30.22	}& \textbf{	30.36	}&{	28.98	}& \textbf{	29.00	}&{	29.68	}& \textbf{	29.80	}&{	24.57	}& \textbf{	24.59	}&{	29.09	}& \textbf{	29.14	}&{	28.51	}& \textbf{	28.58	}&	0.07	\\ \hline		
{75	}&{	0.43	}&{ $	1.04	\sigma$	}&{	28.09	}& \textbf{	28.30	}&{	27.15	}& \textbf{	27.21	}&{	27.73	}& \textbf{	27.95	}&{	22.84	}& \textbf{	22.88	}&{	27.09	}& \textbf{	27.11	}&{	26.58	}& \textbf{	26.69	}&	0.11	\\ \hline		
{100	}&{	0.43	}&{ $	1.11	\sigma$	}&{	26.16	}& \textbf{	26.42	}&{	25.77	}& \textbf{	25.82	}&{	25.74	}& \textbf{	25.93	}&{	21.56	}& \textbf{	21.67	}&{	25.72	}& \textbf{	25.81	}&{	24.99	}& \textbf{	25.13	}&	0.14	\\ \hline	\hline


\multicolumn{16}{|c|}{EPLL \cite{zoran2011learning}}  \\  \cline{1-16}																																		
\multirow{2}[0]{*}{$ \sigma $} & \multicolumn{2}{c||}{{SOS params}} & \multicolumn{2}{c||}{\textsf{Foreman}} & \multicolumn{2}{c||}{\textsf{Lena}} & \multicolumn{2}{c||}{\textsf{House}} & \multicolumn{2}{c||}{\textsf{Fingerprint}} & \multicolumn{2}{c||}{\textsf{Peppers}} &  \multicolumn{3}{c|}{Average} \\ \cline{2-16}																																		
& $ \rho $ & $ \hhsigma $& Orig & SOS      & Orig & SOS      & Orig & SOS      & Orig & SOS      & Orig & SOS & Orig & SOS & Imprv.\\ \hline																																		
{10	}&{	0.09	}&{ $	1.11	\sigma$	}&{	36.98	}& \textbf{	37.09	}&{	35.53	}& \textbf{	35.66	}&{	35.67	}& \textbf{	35.73	}&{	32.12	}& \textbf{	32.33	}&{	34.82	}& \textbf{	34.93	}&{	35.02	}& \textbf{	35.15	}&	0.13	\\ \hline		
{20	}&{	0.09	}&{ $	1.11	\sigma$	}&{	33.70	}& \textbf{	34.03	}&{	32.57	}& \textbf{	32.77	}&{	33.06	}& \textbf{	33.33	}&{	28.26	}& \textbf{	28.49	}&{	32.48	}& \textbf{	32.69	}&{	32.01	}& \textbf{	32.26	}&	0.25	\\ \hline		
{25	}&{	0.18	}&{ $	1.11	\sigma$	}&{	32.44	}& \textbf{	32.78	}&{	31.62	}& \textbf{	31.84	}&{	32.07	}& \textbf{	32.38	}&{	27.14	}& \textbf{	27.30	}&{	31.59	}& \textbf{	31.87	}&{	30.97	}& \textbf{	31.23	}&	0.26	\\ \hline		
{50	}&{	0.43	}&{ $	1.11	\sigma$	}&{	29.24	}& \textbf{	29.60	}&{	28.39	}& \textbf{	28.66	}&{	28.78	}& \textbf{	29.24	}&{	23.63	}& \textbf{	23.69	}&{	28.67	}& \textbf{	29.00	}&{	27.74	}& \textbf{	28.04	}&	0.30	\\ \hline		
{75	}&{	0.43	}&{ $	1.11	\sigma$	}&{	27.17	}& \textbf{	27.55	}&{	26.53	}& \textbf{	26.85	}&{	26.78	}& \textbf{	27.28	}&{	21.51	}& \textbf{	21.54	}&{	26.73	}& \textbf{	27.10	}&{	25.74	}& \textbf{	26.06	}&	0.32	\\ \hline		
{100	}&{	0.43	}&{ $	1.11	\sigma$	}&{	25.58	}& \textbf{	25.91	}&{	25.23	}& \textbf{	25.49	}&{	25.08	}& \textbf{	25.47	}&{	19.77	}&{	19.77	}&{	25.36	}& \textbf{	25.73	}&{	24.20	}& \textbf{	24.47	}&	0.27	\\ \hline		
\end{tabular}%
\end{center}

\label{tab:sos}%
\end{table}%

\begin{figure}[tbp]
\centering
\mbox{\subfigure[{Noisy Image}]{\epsfig{figure=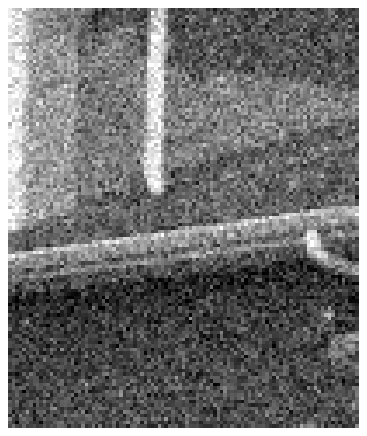, width = 1.19in}} 
 \subfigure[{KSVD, 31.20}]{\epsfig{figure = 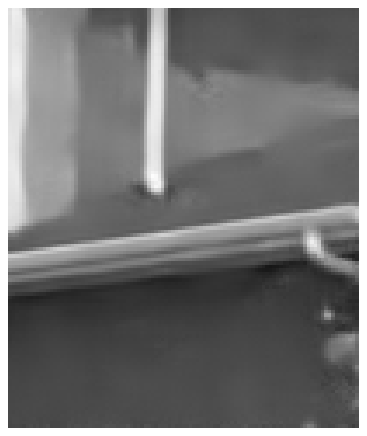, width = 1.19in}}
 \subfigure[{NLM, 30.02}]{\epsfig{figure=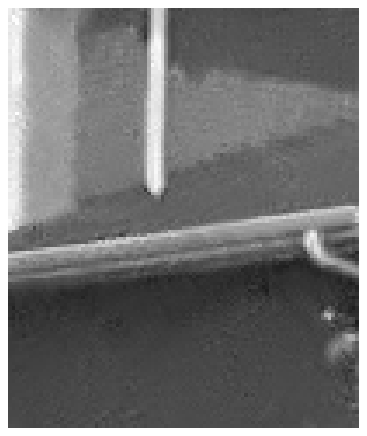, width = 1.19in}}
 \subfigure[{BM3D, 31.88}]{\epsfig{figure=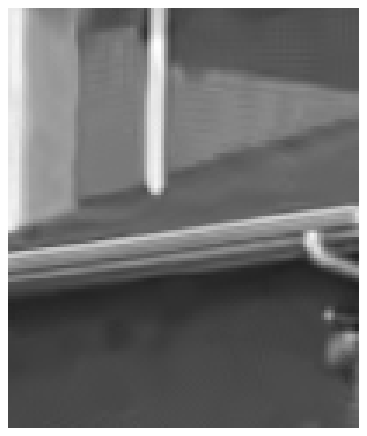, width = 1.19in}}
 \subfigure[{EPLL, 30.88}]{\epsfig{figure=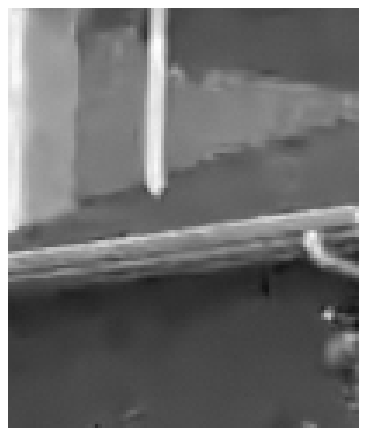, width = 1.19in}}}\\
\centering 
 \mbox{\subfigure[Algo. \ref{consensus_algo}, 31.85]{\epsfig{figure=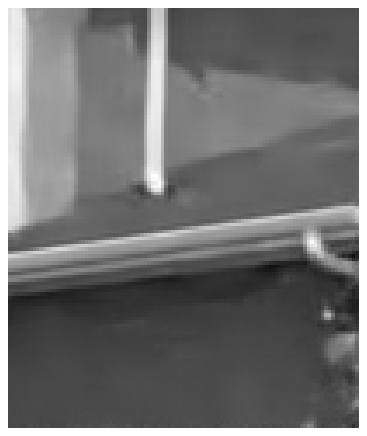, width = 1.19in}}
 \subfigure[SOS KSVD,31.91]{\epsfig{figure = 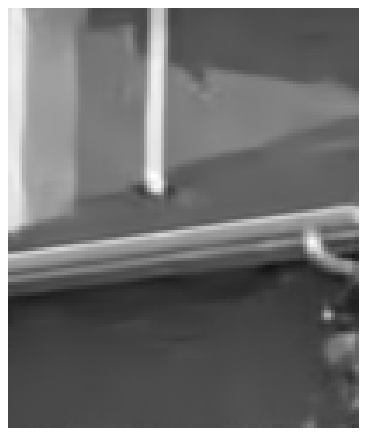, width = 1.19in}} 
 \subfigure[SOS NLM, 30.56]{\epsfig{figure=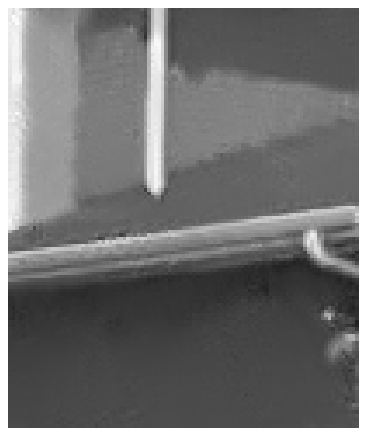, width = 1.19in}}
 \subfigure[SOS BM3D, 31.94]{\epsfig{figure=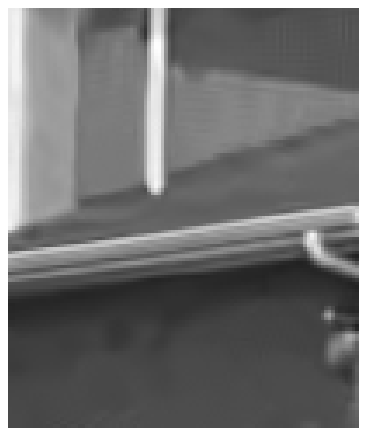, width = 1.19in}}
 \subfigure[SOS EPLL, 31.15]{\epsfig{figure=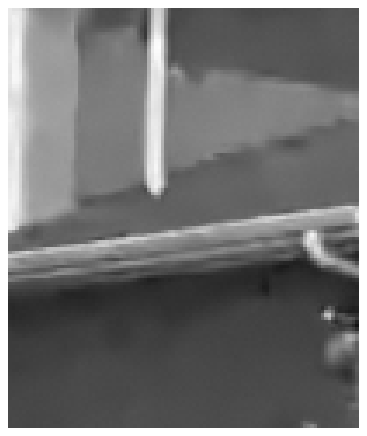, width = 1.19in}}}
\caption{Visual and PSNR comparisons between standard denoising and boosting outcomes of a $ 100 \times 120 $ cropped region from noisy image \textsf{House} ($ \sigma = 25 $).}
\label{clean_images}
\end{figure}

\begin{figure}[tbp]
\centering
\mbox{\subfigure[{Noisy Image}]{\epsfig{figure=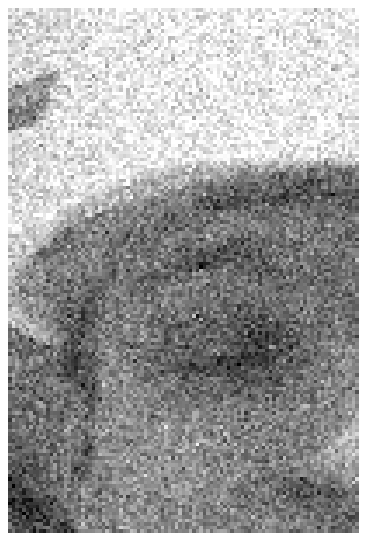, width = 1.19in}} 
 \subfigure[{KSVD, 33.72}]{\epsfig{figure = 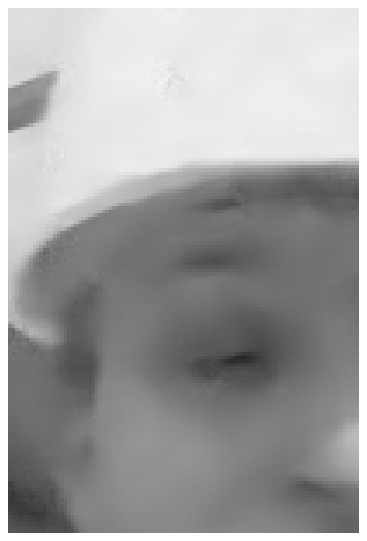, width = 1.19in}}
 \subfigure[{NLM, 31.64}]{\epsfig{figure=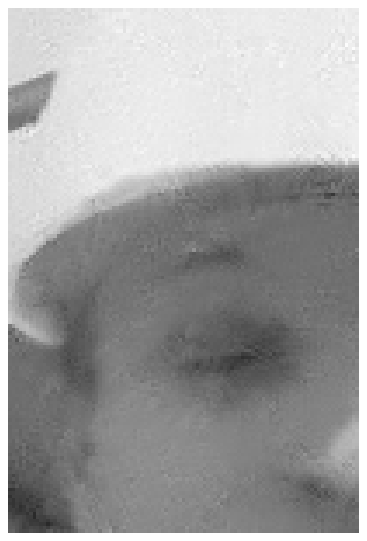, width = 1.19in}}
 \subfigure[{BM3D, 34.66}]{\epsfig{figure=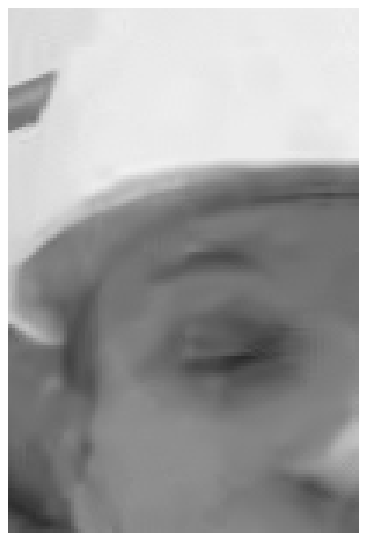, width = 1.19in}}
 \subfigure[{EPLL, 33.62}]{\epsfig{figure=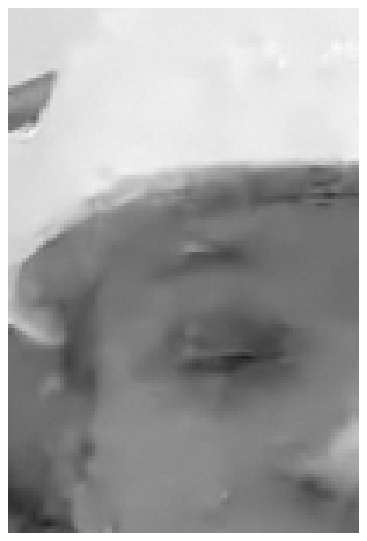, width = 1.19in}}}\\
\centering 
 \mbox{\subfigure[Algo. \ref{consensus_algo}, 34.37]{\epsfig{figure=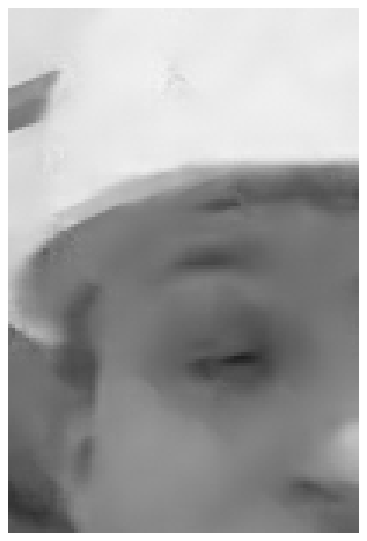, width = 1.19in}}
 \subfigure[SOS KSVD,34.38]{\epsfig{figure = 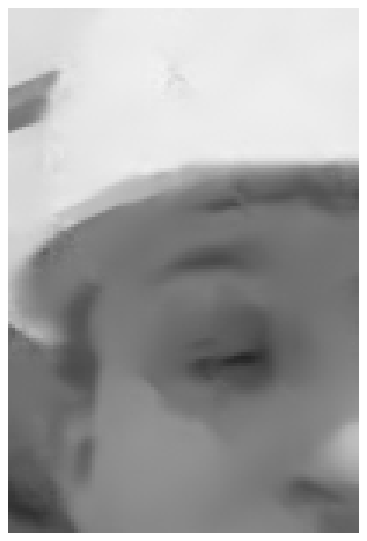, width = 1.19in}} 
 \subfigure[SOS NLM, 32.28]{\epsfig{figure=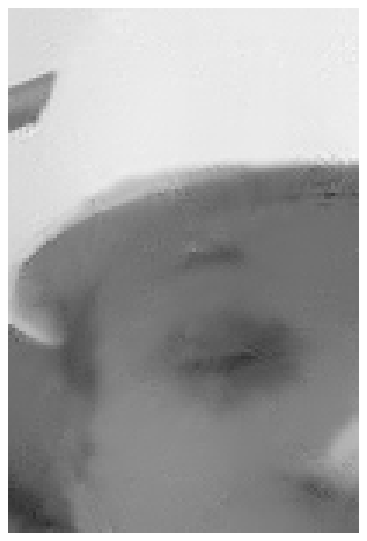, width = 1.19in}}
 \subfigure[SOS BM3D, 34.71]{\epsfig{figure=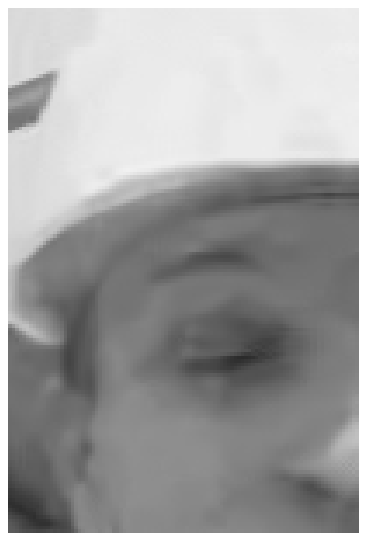, width = 1.19in}}
 \subfigure[SOS EPLL, 34.09]{\epsfig{figure=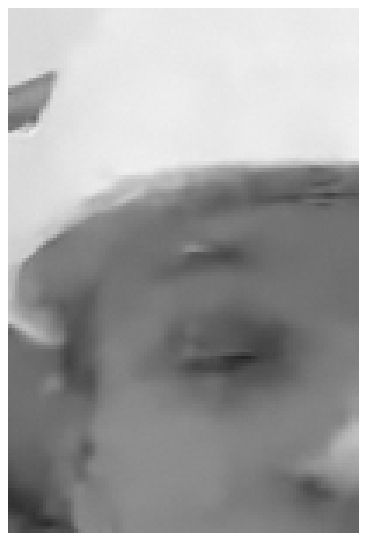, width = 1.19in}}}
\caption{Visual and PSNR comparisons between standard denoising and boosting outcomes of a $ 100 \times 150 $ cropped region from noisy image \textsf{Foreman} ($ \sigma = 25 $).}
\label{clean_images2}
\end{figure}
\begin{figure}[tbp]
\centering
\mbox{\subfigure[{Noisy Image}]{\epsfig{figure=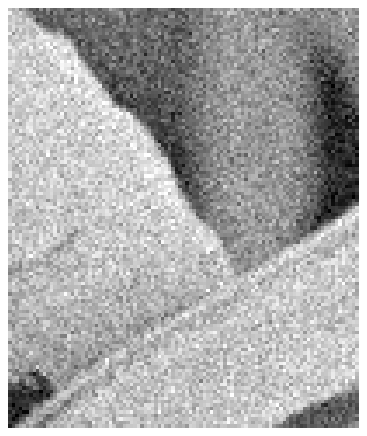, width = 1.19in}} 
 \subfigure[{KSVD, 33.06}]{\epsfig{figure = 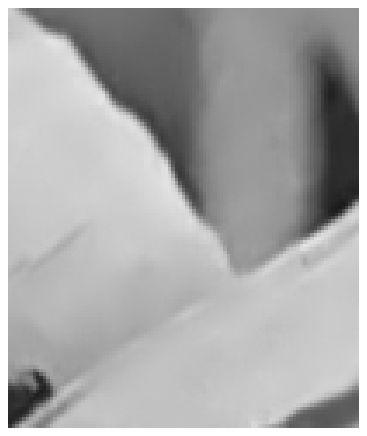, width = 1.19in}}
 \subfigure[{NLM, 32.13}]{\epsfig{figure=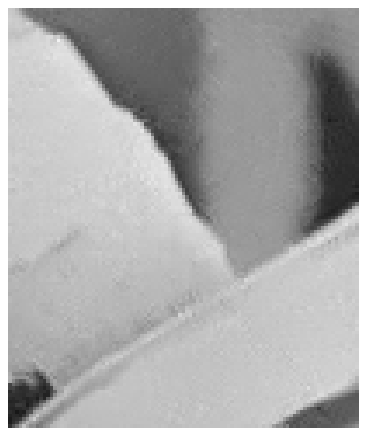, width = 1.19in}}
 \subfigure[{BM3D, 34.03}]{\epsfig{figure=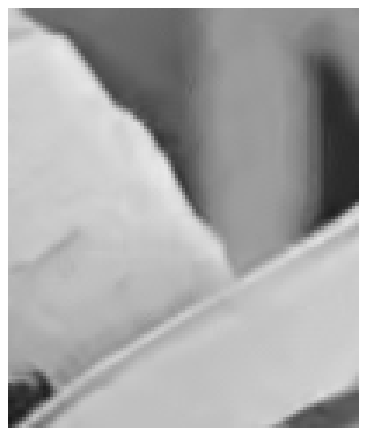, width = 1.19in}}
 \subfigure[{EPLL, 33.29}]{\epsfig{figure=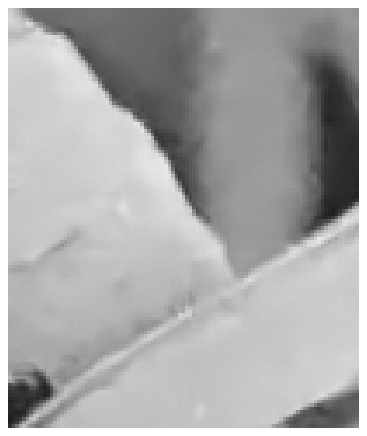, width = 1.19in}}}\\
\centering 
 \mbox{\subfigure[Algo. \ref{consensus_algo}, 33.56]{\epsfig{figure=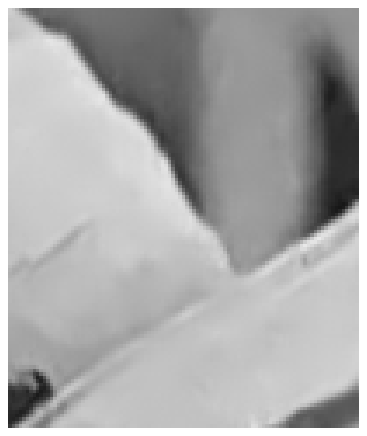, width = 1.19in}}
 \subfigure[SOS KSVD,33.48]{\epsfig{figure = 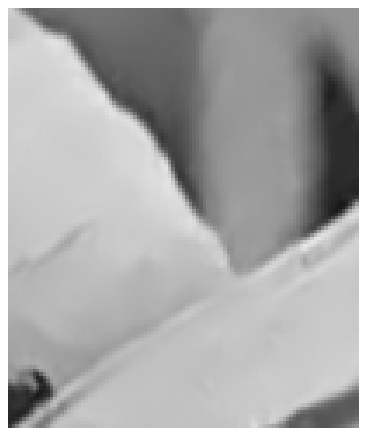, width = 1.19in}} 
 \subfigure[SOS NLM, 32.41]{\epsfig{figure=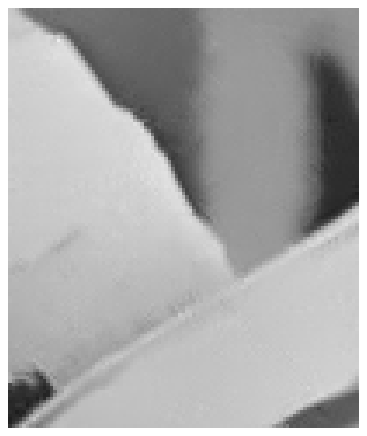, width = 1.19in}}
 \subfigure[SOS BM3D, 34.07]{\epsfig{figure=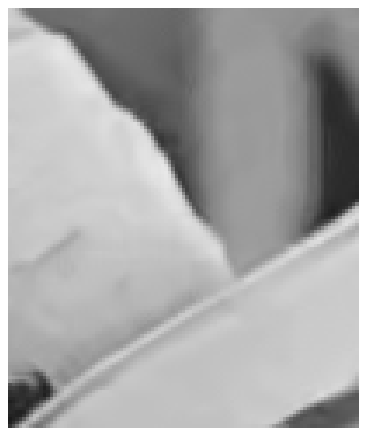, width = 1.19in}}
 \subfigure[SOS EPLL, 33.56]{\epsfig{figure=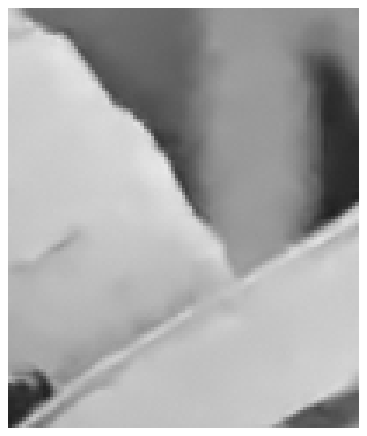, width = 1.19in}}}
\caption{Visual and PSNR comparisons between standard denoising and boosting outcomes of a $ 100 \times 120 $ cropped region from noisy image \textsf{Lena} ($ \sigma = 20 $).}
\label{clean_images3}
\end{figure}

\begin{figure}[htbp]
\centering
\mbox{{\epsfig{figure=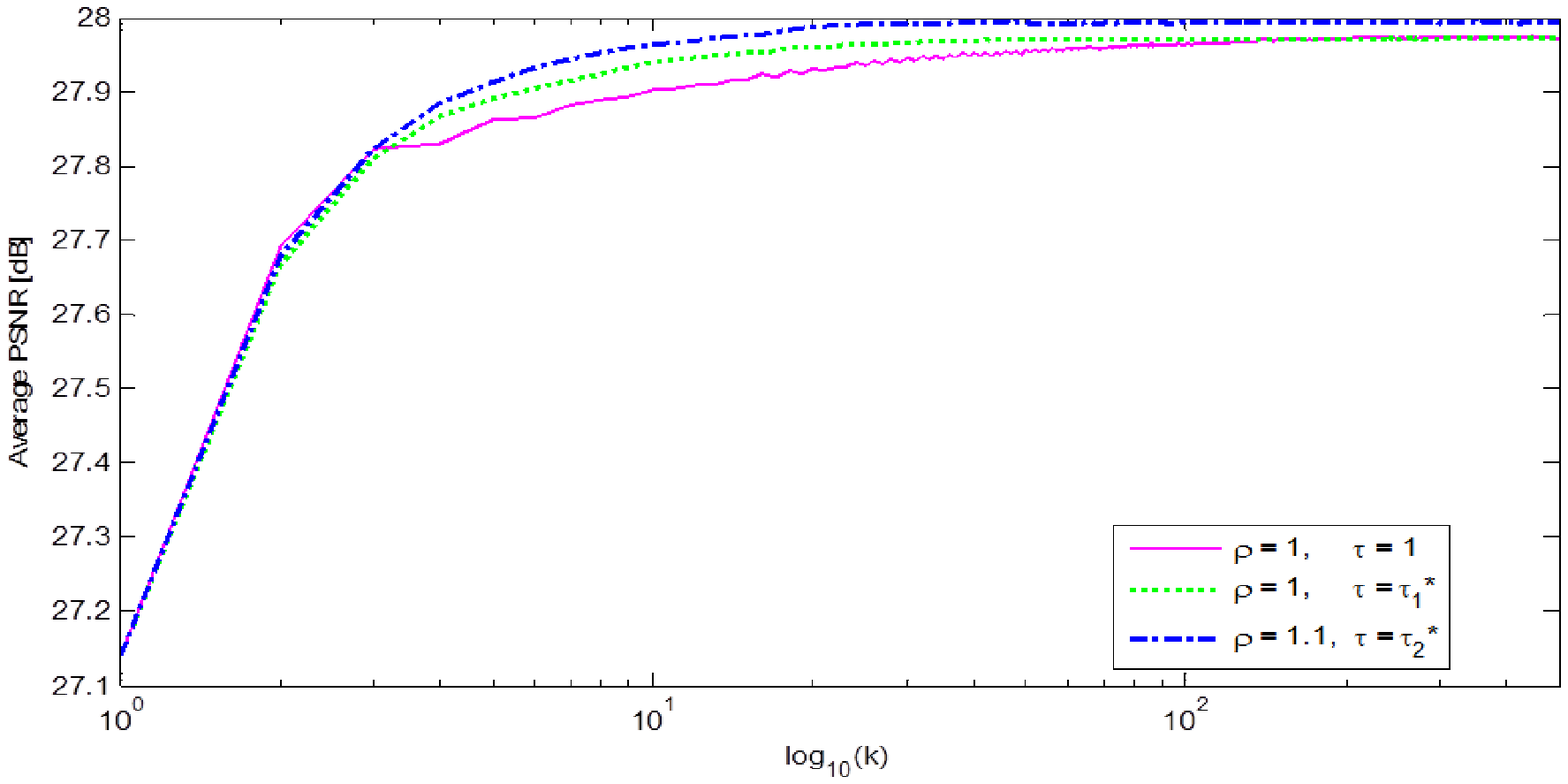, width = 3in}}}
\vspace{-0.2cm}	
\caption{Demonstration of the effect of $ \tau^* $ on the SOS boosting outcome for the K-SVD denoising ($ \sigma=50 $).}
\label{fig:tau}

\end{figure}

Table \ref{tab:sos} lists the restoration results of various denoising algorithms and their SOS versions. The PSNR values that appear in the 'Orig' column are obtained by applying the denoising algorithm on $ \y $ using the input noise level $ \sigma $ (and not $ \hhsigma $ as done at the consecutive SOS-steps). These are also the first estimates of the SOS boosting (i.e., $ \xh^1 $).
In the case of the K-SVD denoising \cite{KSVD_REF1}, at the first SOS-step we apply $ 20 $ iterations of sparse-coding and dictionary-update, while at the rest SOS-steps we apply only $ 2 $ such iterations (we found this to be a convenient compromise between runtime and performance). We operate the K-SVD \cite{KSVD_REF1}, NLM \cite{ipol_nlm}, BM3D \cite{BM3D_REF} and EPLL \cite{zoran2011learning} for $ 30 $, $ 2 $, $ 3 $ and $ 4 $ SOS-steps, respectively.

The 'average imprv.' column in Table \ref{tab:sos} indicates that the SOS boosting achieves an improvement over the original denoising algorithms. More specifically, for all denoising algorithms, images, and noise levels, the SOS outcomes are at least as good as the original results and more important -- usually better (in terms of PSNR). A clear improvement over the whole range of noise levels is achieved for the K-SVD \cite{KSVD_REF1}, NLM \cite{ipol_nlm, NL_DENOISE_REF4} and EPLL \cite{zoran2011learning}. While in the case of the BM3D \cite{BM3D_REF}, we succeed in improving it slightly, mainly for high noise energy. The fact that the BM3D performance is very close to the denoising bound posed in \cite{levin2011natural} explains the difficulties in improving it. A visual comparison is given in Figures \ref{clean_images}, \ref{clean_images2}, and \ref{clean_images3} illustrating the effectiveness of the SOS boosting. Compared to the original results, the SOS offers better restoration of edges (in the case of the K-SVD -- focus on the house's roof, foreman's eye and ear, and lena's hat). In addition, the SOS obtains cleaner estimations (when using the NLM), and less artifacts (for the EPLL and somewhat also for the BM3D).

\begin{table}[ht]																														\caption{Denoising results [PSNR] of  the K-SVD \cite{KSVD_REF1} and its SOS boosting outcomes, where we use the parameter  $\tau^*$ (according to Equation (\ref{tau_star})), along with the appropriate $ \rho $ and $ \hhsigma $. The best results per each image and noise level are highlighted.}																																	
\begin{center}\footnotesize
\renewcommand{\arraystretch}{1.2}
\renewcommand{\tabcolsep}{0.022cm}

\begin{tabular}{|c||c|c||c|c||c|c||c|c||c|c||c|c||c|c|c|} \hline																																	
																																	

\multirow{2}[0]{*}{$ \sigma $} & \multicolumn{2}{c||}{{SOS params}} & \multicolumn{2}{c||}{\textsf{Foreman}} & \multicolumn{2}{c||}{\textsf{Lena}} & \multicolumn{2}{c||}{\textsf{House}} & \multicolumn{2}{c||}{\textsf{Fingerprint}} & \multicolumn{2}{c||}{\textsf{Peppers}} &  \multicolumn{3}{c|}{Average} \\ \cline{2-16}																																	
& $ \rho $ & $ \hhsigma $& Orig & SOS      & Orig & SOS      & Orig & SOS      & Orig & SOS      & Orig & SOS & Orig & SOS & Imprv.\\ \hline																																	
{10	}&{	0.30	}&{ $	1.00	\sigma$	}&{	36.92	}& \textbf{	37.14	}&{	35.47	}& \textbf{	35.58	}&{	36.25	}& \textbf{	36.49	}&{	32.27	}& \textbf{	32.35	}&{	34.68	}& \textbf{	34.72	}&{	35.12	}& \textbf{	35.26	}&	0.14	\\ \hline	
{20	}&{	0.60	}&{ $	1.00	\sigma$	}&{	33.81	}& \textbf{	34.11	}&{	32.43	}& \textbf{	32.68	}&{	33.34	}& \textbf{	33.62	}&{	28.31	}& \textbf{	28.54	}&{	32.29	}& \textbf{	32.35	}&{	32.04	}& \textbf{	32.26	}&	0.22	\\ \hline	
{25	}&{	1.00	}&{ $	1.00	\sigma$	}&{	32.83	}& \textbf{	33.12	}&{	31.32	}& \textbf{	31.65	}&{	32.39	}& \textbf{	32.74	}&{	27.13	}& \textbf{	27.46	}&{	31.43	}& \textbf{	31.53	}&{	31.02	}& \textbf{	31.30	}&	0.28	\\ \hline	
{50	}&{	1.10	}&{ $	1.00	\sigma$	}&{	28.88	}& \textbf{	29.86	}&{	27.75	}& \textbf{	28.42	}&{	28.01	}& \textbf{	29.05	}&{	23.20	}& \textbf{	24.03	}&{	28.16	}& \textbf{	28.68	}&{	27.20	}& \textbf{	28.01	}&	0.81	\\ \hline	
{75	}&{	1.20	}&{ $	1.00	\sigma$	}&{	26.24	}& \textbf{	27.36	}&{	25.74	}& \textbf{	26.50	}&{	25.23	}& \textbf{	27.08	}&{	19.93	}& \textbf{	22.02	}&{	25.73	}& \textbf{	26.80	}&{	24.57	}& \textbf{	25.95	}&	1.38	\\ \hline	
{100	}&{	1.20	}&{ $	1.00	\sigma$	}&{	25.21	}& \textbf{	25.46	}&{	24.50	}& \textbf{	25.09	}&{	23.69	}& \textbf{	24.70	}&{	17.98	}& \textbf{	19.93	}&{	24.17	}& \textbf{	25.17	}&{	23.11	}& \textbf{	24.07	}&	0.96	\\ \hline

\end{tabular}%
\end{center}																																	
																																	
\label{tab:tau}%
\vspace{-0.15cm}																															
\end{table}%

In the context of the K-SVD denoising, based on Equation (\ref{tau_star}), we demonstrate the effect of $ \tau^{*} $ on the SOS recursive function. Note that we do not test its influence on the other denoising algorithms because the information about their eigenvalues range, which is required in Equation (\ref{tau_star}), is not derived.
Figure \ref{fig:tau} plots the average PSNR over the test images ($ \sigma = 50 $), as a function of the SOS-step, for 3 different parameter settings: First, as a baseline, we apply the SOS with $ \rho=1 $ and $ \tau=1 $ (without using the closed-form expression for $ \tau^* $). Second, we improve the convergence rate by using $ \tau^* $ with the same signal-emphasis factor ($ \rho=1 $). Third, we plot the PSNR that obtained by the couple that leads to the best restoration ($ \rho=1.1 $ with the corresponding $ \tau^* $). As a reminder, according to Section \ref{parametrization} and Appendix \ref{best_params}, the parameters $ \rho $ and $ \tau $ affect the conditions for convergence and its rate. More specifically, a modification of $ \rho $ without an adjustment of $ \tau $ may violate the condition for convergence (e.g. according to condition (\ref{tau_rho_condition}), the couple $ \rho=1.1 $ and $ \tau = 1 $ results in $ \gamma>1 $). Therefore, using $ \tau^* $ enables to modify $ \rho $ and still converge, even with the fastest rate. These results are consistent with Table \ref{tab:tau}, which lists the achieved PSNR when applying the SOS for 30 steps using the best $ \rho $ and $ \tau^* $ (per noise level). As can be seen, $ \tau^* $ not only results in a faster convergence, but also allows a stronger emphasis of the estimated signal, thus leading to better restoration. 

\begin{table}[!htbp]																								
\caption{Comparison between the denoising results [PSNR] of the original K-SVD algorithm \cite{KSVD_REF1} and its ''sharing the disagreement'' boosting outcome (Algorithm \ref{consensus_algo}). The best results per each image and noise level are highlighted.}																															
\begin{center}\footnotesize
\renewcommand{\arraystretch}{1.2}
\renewcommand{\tabcolsep}{0.03cm}
\begin{tabular}{|c||c||c|c||c|c||c|c||c|c||c|c||c|c|c|} \hline																															
\multirow{2}[0]{*}{$ \sigma $} & \multirow{2}[0]{*}{{$ \hhsigma $}} & \multicolumn{2}{c||}{\textsf{Foreman}} & \multicolumn{2}{c||}{\textsf{Lena}} & \multicolumn{2}{c||}{\textsf{House}} & \multicolumn{2}{c||}{\textsf{Fingerprint}} & \multicolumn{2}{c||}{\textsf{Peppers}} &  \multicolumn{3}{c|}{Average} \\ \cline{3-15}																															
& & Orig & Boost      & Orig & Boost      & Orig & Boost      & Orig & Boost      & Orig & Boost & Orig & Boost & Imprv. \\ \hline																															
{10	}&{ $	1.08	\sigma$	}&{	36.92	}& \textbf{	37.13	}&{	35.47	}& \textbf{	35.58	}&{	36.25	}& \textbf{	36.34	}&{	32.27	}& \textbf{	32.35	}&{	34.68	}& \textbf{	34.70	}&{	35.12	}& \textbf{	35.22	}&	0.10	\\ \hline	
{20	}&{ $	1.02	\sigma$	}&{	33.81	}& \textbf{	34.11	}&{	32.43	}& \textbf{	32.68	}&{	33.34	}& \textbf{	33.56	}&{	28.31	}& \textbf{	28.59	}&{	32.29	}& \textbf{	32.37	}&{	32.04	}& \textbf{	32.26	}&	0.22	\\ \hline	
{25	}&{ $	1.02	\sigma$	}&{	32.83	}& \textbf{	33.17	}&{	31.32	}& \textbf{	31.64	}&{	32.39	}& \textbf{	32.71	}&{	27.13	}& \textbf{	27.47	}&{	31.43	}& \textbf{	31.60	}&{	31.02	}& \textbf{	31.32	}&	0.30	\\ \hline	
{50	}&{ $	1.00	\sigma$	}&{	28.88	}& \textbf{	29.37	}&{	27.75	}& \textbf{	28.28	}&{	28.01	}& \textbf{	28.67	}&{	23.20	}& \textbf{	24.04	}&{	28.16	}& \textbf{	28.55	}&{	27.20	}& \textbf{	27.78	}&	0.58	\\ \hline	
{75	}&{ $	1.00	\sigma$	}&{	26.24	}& \textbf{	27.04	}&{	25.74	}& \textbf{	26.28	}&{	25.23	}& \textbf{	26.54	}&{	19.93	}& \textbf{	21.76	}&{	25.73	}& \textbf{	26.52	}&{	24.57	}& \textbf{	25.63	}&	1.06	\\ \hline	
{100	}&{ $	1.00	\sigma$	}&{	25.21	}& \textbf{	25.28	}&{	24.50	}& \textbf{	24.91	}&{	23.69	}& \textbf{	24.43	}&{	17.98	}& \textbf{	19.82	}&{	24.17	}& \textbf{	24.92	}&{	23.11	}& \textbf{	23.87	}&	0.76	\\ \hline	
\end{tabular}%
\end{center}																							
\label{tab:cons}%
\vspace{-0.18cm}
\end{table}%

\subsection{Sharing the disagreement}
\label{sharing_sec}
We demonstrate the effectiveness of the local-global interpretation of the SOS boosting, which was described in Section \ref{disagreement_algo}. The denoising results of Table \ref{tab:cons} are obtained by applying Algorithm \ref{consensus_algo} for $ 30 $ steps, where each step includes $ 2 $ sparse-coding and dictionary-update iterations. The initial dictionary is obtained by applying $ 20 $ iterations of the K-SVD algorithm. Similarly to the SOS boosting, we tune the parameter $ \hhsigma $ per each input $ \sigma $ (this variant is limited to $ \rho=1 $ and $ \tau = 1 $).

According to Table \ref{tab:cons}, for all images and noise levels, ''sharing the disagreement'' boosting achieves a clear improvement over the original K-SVD algorithm \cite{KSVD_REF1}. Notice the resemblance and the differences in the PSNR values between Table \ref{tab:cons} and the K-SVD part in Table \ref{tab:sos}. In general, the differences originate from the non-linearity of the denoising algorithm -- the input patch to the sparse-coding step is different between the SOS and its local-global variant. 
As a reminder, the equivalence between these two approaches is valid under the assumption of a fixed filter-matrix (see Appendix \ref{sos_and_disagree} for more details). Furthermore, in the case of SOS boosting, more freedom is obtained by tuning the parameters $ \rho $ and $ \tau $, which may lead to better utilization of the prior (as shown in Figure \ref{fig:tau} and Table \ref{tab:tau}). However, visually, according to Figure \ref{clean_images}, \ref{clean_images2}, and \ref{clean_images3}, the outcomes of the SOS and its local-global variant are very similar, both of them improve effectively the restoration of the underlying signal.

To conclude, we demonstrate the potential of the SOS-boosting and its local-global interpretation. The proposed algorithm achieves a clear and meaningful improvement over the examined state-of-the-art denoising algorithms, both visually in terms of PSNR.

\section{Conclusions and Future Directions}
\label{conclusions}
We have presented the SOS boosting -- a generic method for improving various image denoising algorithms. The improvement is achieved by treating the denoiser as a ''black-box'' and repeating 3 simple SOS-steps: (i) \emph{Strengthening} the signal,
(ii) \emph{Operating} the denoising algorithm, and (iii) \emph{Subtracting} the previous denoised image from the result. In addition, we provided an interesting local-global interpretation, called ''sharing the disagreement'' boosting, indicating that the SOS boosting not only leverages the improved SNR of the estimates, but also reduces the gap between the local patch processing and the global need for a whole denoised image, all in the context of the K-SVD denoising algorithm. We also constructed the matrix-formulation of the K-SVD (and similar algorithms), showing that its eigenvalues are in the range of 0 to 1. Under these conditions, we have studied the convergence of the SOS boosting recursive function, leading to the conclusion that for various known denoising algorithms, the SOS boosting is guaranteed to converge. Moreover, a generalization of the SOS function has been obtained by introducing two parameters that govern the steady-state result, soften the requirements for convergence (the eigenvalues range) and the rate-of-convergence. We also provided a closed-form expression for the parameter that leads to the fastest convergence.

Finally, we have introduced a graph-based interpretation, showing that the SOS boosting acts as a graph Laplacian regularization method, thus effectively estimating the structure of the underlying signal. Inspired by the SOS scheme, we suggested novel recursive algorithms that treat the denoiser as a ''black-box'' in order to minimize related graph Laplacian objective functions, without explicitly constructing the weighted graph.

The proposed algorithm is easy to use, it reduces the local-global gap, acting as a graph Laplacian regularizer, it is applicable to a wide range of denoising algorithms and it converges -- these make it a powerful and convenient tool for improving various denoising algorithms, as demonstrated in the experiments.

It is intriguing to study the proposed iterative algorithms that are defined in Section \ref{graph} (which minimize different graph Laplacian cost functions). They may lead to better results than the original methods due to the non-linearity of the denoiser and its adaptivity to the signal-strengthened image. We hope that other restoration problems, such as super-resolution \cite{NL_SR_REF}, interpolation/inpainting \cite{romano_interp} and more, could also benefit from a similar concept.

\appendix

\section{Periodic Boundary Condition}
\label{periodic}
Following Equation (\ref{denosied_image_filtered}), the periodic boundary condition affects the term $ \sum_{i=1}^{N} \R_i^T\R_i $, which is a diagonal matrix that counts the number of appearances per pixel in the final patch-averaging. Due to boundary effects, the values along the diagonal in this matrix are different (since the number of overlapping patches in the image borders is smaller than in other areas). As shown in Appendix \ref{Wproperties}, the numerator of Equation (\ref{denosied_image_filtered}) is a symmetric and positive definite matrix. Each row of this matrix is normalized by the number of overlapping patches, i.e., by the corresponding diagonal element from $ ( \mu \I + \sum_{i=1}^{N} \R_i^T\R_i ) $. As a result, the rows and columns are normalized by different constants which ruin the symmetric property. By assuming periodic boundary condition, all the pixels have the same number of representations, which equals to the patch size. In this case, we get $ \sum_{i=1}^{N} \R_i^T\R_i = n\I $, where $ n $ is the patch size and thus, the rows and columns are normalized by the same constant, which preserves the symmetric property of $ \W $.

\section{Properties of the K-SVD Filter-Matrix}
\label{Wproperties}
{\em Proof of property 1: symmetric $ \W=\W^T $.}
Following Equation (\ref{denosied_image_filtered}) and based on the assumption of periodic boundary condition (see Appendix \ref{periodic}) the matrix $ \W $ can be expressed as
\begin{align}
\label{sym1}
\W & { = \left( \mu \I + \sum_{i=1}^{N} \R_i^T\R_i \right)^{-1} \left( \mu \I + \sum_{i=1}^{N} \R_i^T \D_{s_i}\left(\D_{s_i}^{T}\D_{s_i}\right)^{-1}\D_{s_i}^{T} \R_i \right) } \\ \notag
   & { = \left( \mu \I + n\I \right)^{-1} \left( \mu \I + \sum_{i=1}^{N} \R_i^T \D_{s_i}\left(\D_{s_i}^{T}\D_{s_i}\right)^{-1}\D_{s_i}^{T} \R_i \right)} \\ \notag
   & { = \frac{1}{\mu + n}\left( \mu \I + \sum_{i=1}^{N} \R_i^T \D_{s_i}\left(\D_{s_i}^{T}\D_{s_i}\right)^{-1}\D_{s_i}^{T} \R_i \right)}. \notag
\end{align}
Notice that the term $ \Z_i = \R_i^T \D_{s_i}(\D_{s_i}^{T}\D_{s_i})^{-1}\D_{s_i}^{T} \R_i $ is symmetric (in fact, it is also positive semi-definite (PSD), as it is built of $ \R_i^T \hat{\Z}_i \R_i $, where $ {\hat \Z}_i $ is a projection matrix \cite{horn2012matrix}). Thus, $ \W $ is built as a sum of $ N+1 $ matrices, each of them symmetric, which leads to the claimed symmetry, $ \W=\W^T $. \endproof

{\em Proof of properties 2 \& 3:} \emph{positive definite $ \W \succ 0 $ and \mbox{$ \lambda_{min}(\W) \geq \frac{\mu}{\mu + n} $}}.
As we have seen above, $ \W $ can be written as
\begin{align}
\label{WandZ}
\W & { = \frac{\mu}{\mu + n}\I + \frac{1}{\mu + n}\sum_{i=1}^{N} \left(\R_i^T \D_{s_i}\left(\D_{s_i}^{T}\D_{s_i}\right)^{-1}\D_{s_i}^{T} \R_i\right)} \\ \notag
   & { = \frac{\mu}{\mu + n}\I + \frac{1}{\mu + n}\sum_{i=1}^{N} \Z_i} \\ \notag
   & { = \b\I + \A}, \notag
\end{align}
where $ \A = \frac{1}{\mu + n}\sum_{i=1}^{N} \Z_i$, and $ \b = \frac{\mu}{\mu + n}>0 $.
As mentioned above, $ \Z_i $ is PSD and therefore, according to Equation (\ref{WandZ}), $ \A $ is a linear combination of $ \Z_i \succeq \textbf{0} $, thus it is PSD as well \cite{horn2012matrix}. Finally, based on the fact that the eigenvalues of $ \A+\b\I $ are lower-bounded by $ \b>0 $, we get that $ \W \succ \textbf{0}$, with minimal eigenvalue satisfying
\begin{align}
\label{minlambda}
\lambda_{min}(\W) \geq \b = \frac{\mu}{\mu + n}.\quad \endproof  
\end{align}

{\em Proof of property 4 (\& 5): $ \W\one = \W^T\one = \one $}. This property originates directly from the K-SVD denoising algorithm, which preserves the DC component of the image. In general, the trained dictionary is adapted to the image patches after their DC is removed. Once trained, the DC is returned as an additional atom $ d_0 $. Thus, this DC atom is necessarily orthogonal to the rest of the dictionary atoms. Each patch is represented by
\begin{align}
\label{dsi}
{\D_{s_i} = [d_0, d_1 ,d_2, ...] = [d_0, \tilde{\D}_{s_i}],}  
\end{align}
where $ d_0 \in \RR^{n} $ is the DC atom (the DC atom is included if the mean of the patch is not zero) and for the rest of the atoms (if any), $ d_i^T d_0 = 0 $, i.e., $ d_0^T{\tilde \D}_{s_i} = \underbar{0} $. Note that the Gram matrix in this case is block-diagonal
\begin{align}
\label{bd}
\D_{s_i}^T\D_{s_i} = \begin{bmatrix}
1 & 0 \\ \underbar{0} & \tilde{\D}_{s_i}^T\tilde{\D}_{s_i}
\end{bmatrix}.
\end{align}
Following Equation (\ref{WandZ}), when multiplying W by a constant image, $ \one $, we get
\begin{align}
\label{w1}
\W\one = \frac{\mu}{\mu + n} \one + \frac{1}{\mu + n}\sum_{i=1}^{N} \Z_i \one.  
\end{align}
Let us look at the term $ \Z_i \one$,
\begin{align}
\label{z1}
\Z_i\one = \R_i^T \D_{s_i}\left(\D_{s_i}^{T}\D_{s_i}\right)^{-1}\D_{s_i}^{T} \R_i \one.
\end{align}
$ \R_i \one = \one$ -- this is a shorter constant vector of length $ n $.
\mbox{$ \D_{s_i}^T \one = [n, 0, 0, ...]^T $} due to the orthogonality of the rest of the atoms to the DC.
Multiplication of the the inverse of $ \D_{s_i}^T \D_{s_i} $ results with $ [n, 0, 0, ...]^T $.
The outcome of \mbox{$ \D_{s_i} [n, 0, 0, ...]^T = nd_0 = \one$} is the desired DC patch of length $ n $. Finally, $ \R_i^T $ returns the resulting constant patch back to its original location in the image. 

Returning to Equation (\ref{w1}), the input image $ \one $ is divided into $ N $ overlapping patches (one per pixel) of length $ n $, where each DC patch is represented perfectly by $ \Z_i $. Since each pixel appears in $ n $ patches, we get that $ \sum_{i=1}^{N} \Z_i \one = n\one $. As a result, \mbox{$ \W \one = \frac{\mu}{\mu + n} \one + \frac{n}{\mu + n}\one = \one$} preserves the DC of the image. Based on the symmetric property of $ \W $, we get that $ \W^T\one = \W\one = \one $. \endproof

\begin{figure}[tbp]
\centering
\mbox{\subfigure[$ \Omega_1 $]{\epsfig{figure=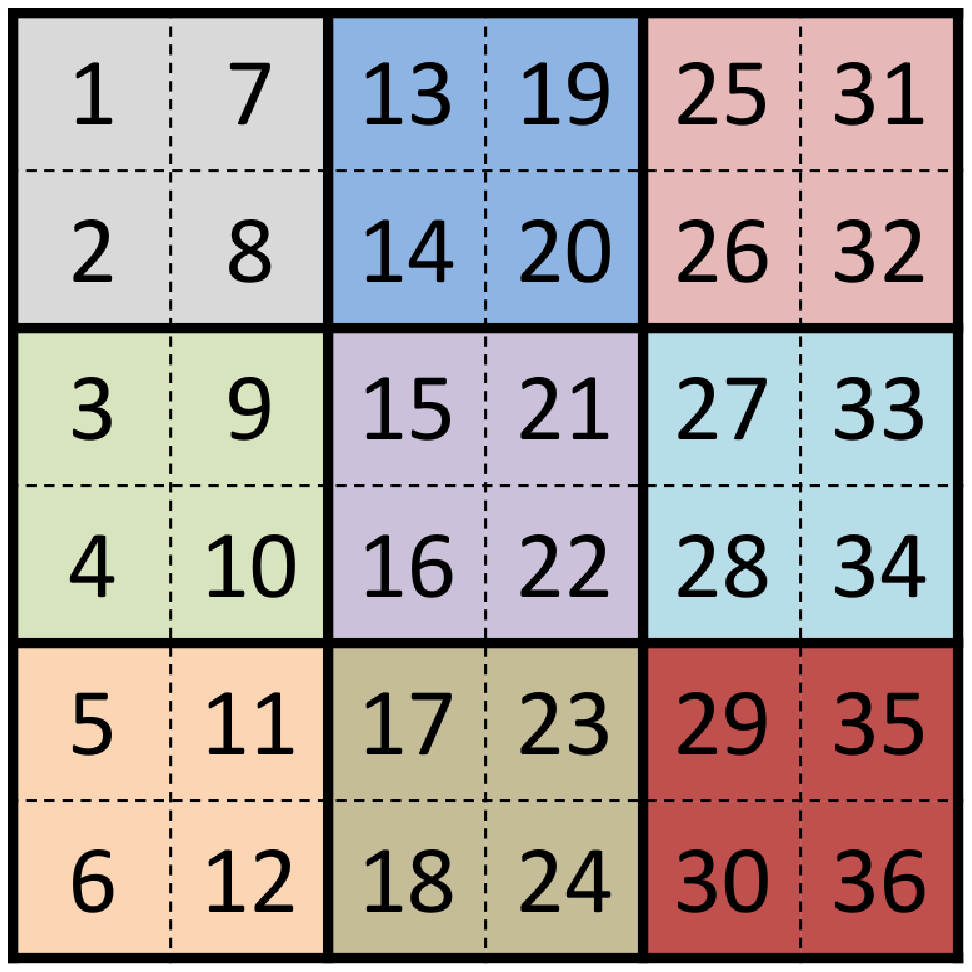, width = 1.5in}} 
 \subfigure[$ \Omega_2 $]{\epsfig{figure = 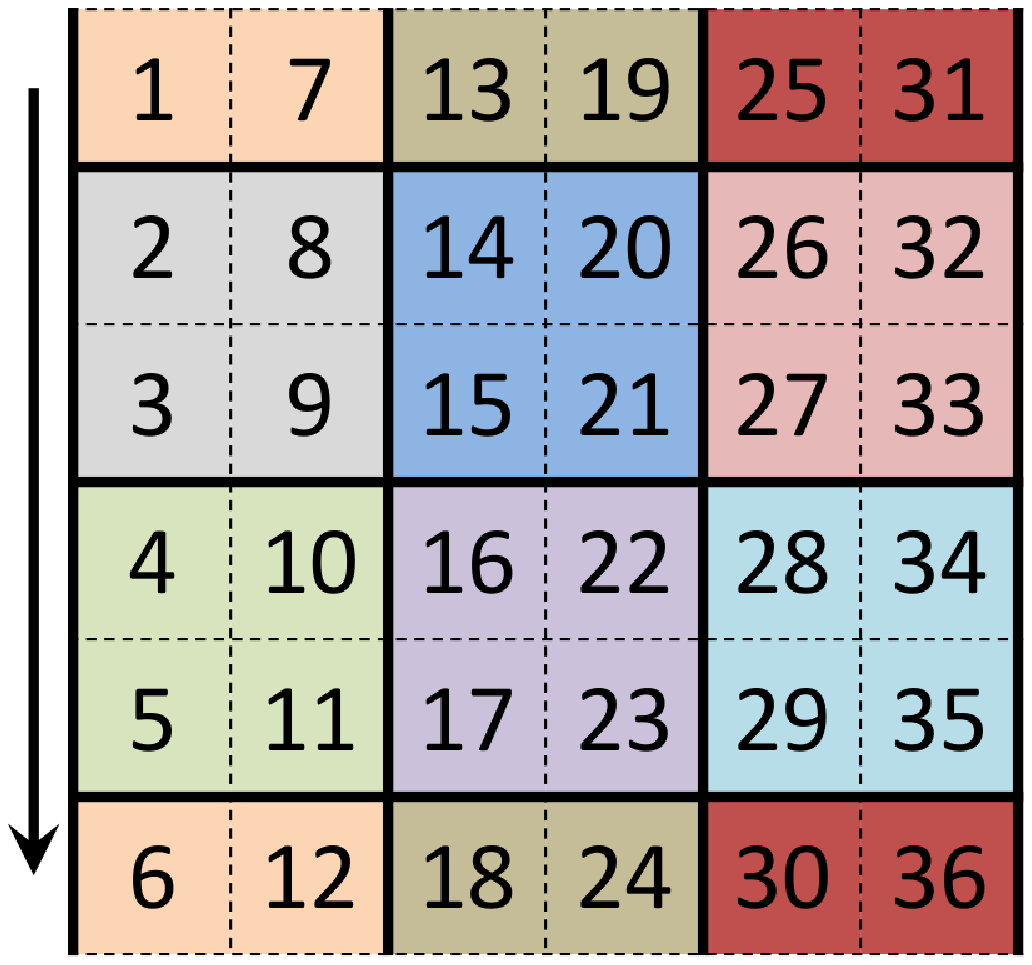, width = 1.5in}}
 \subfigure[$ \Omega_3 $]{\epsfig{figure=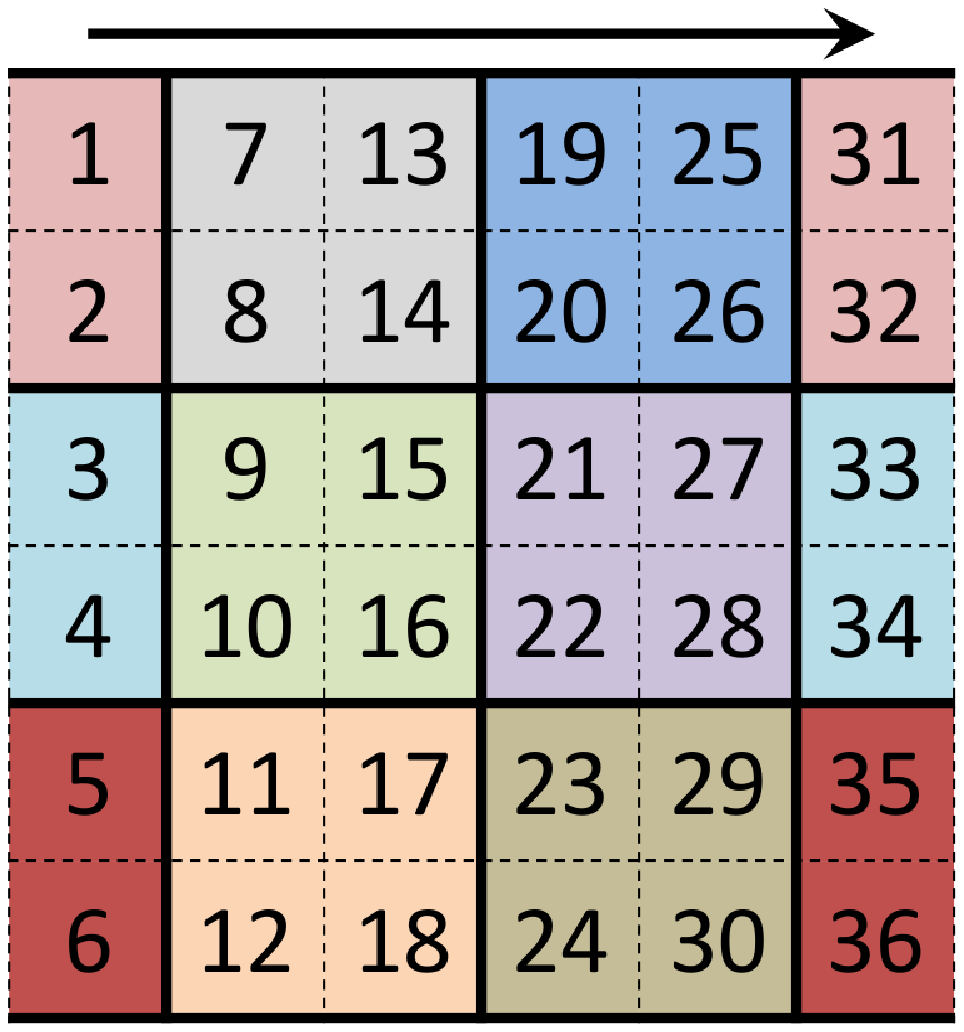, width = 1.5in}}
 \subfigure[$ \Omega_4 $]{\epsfig{figure=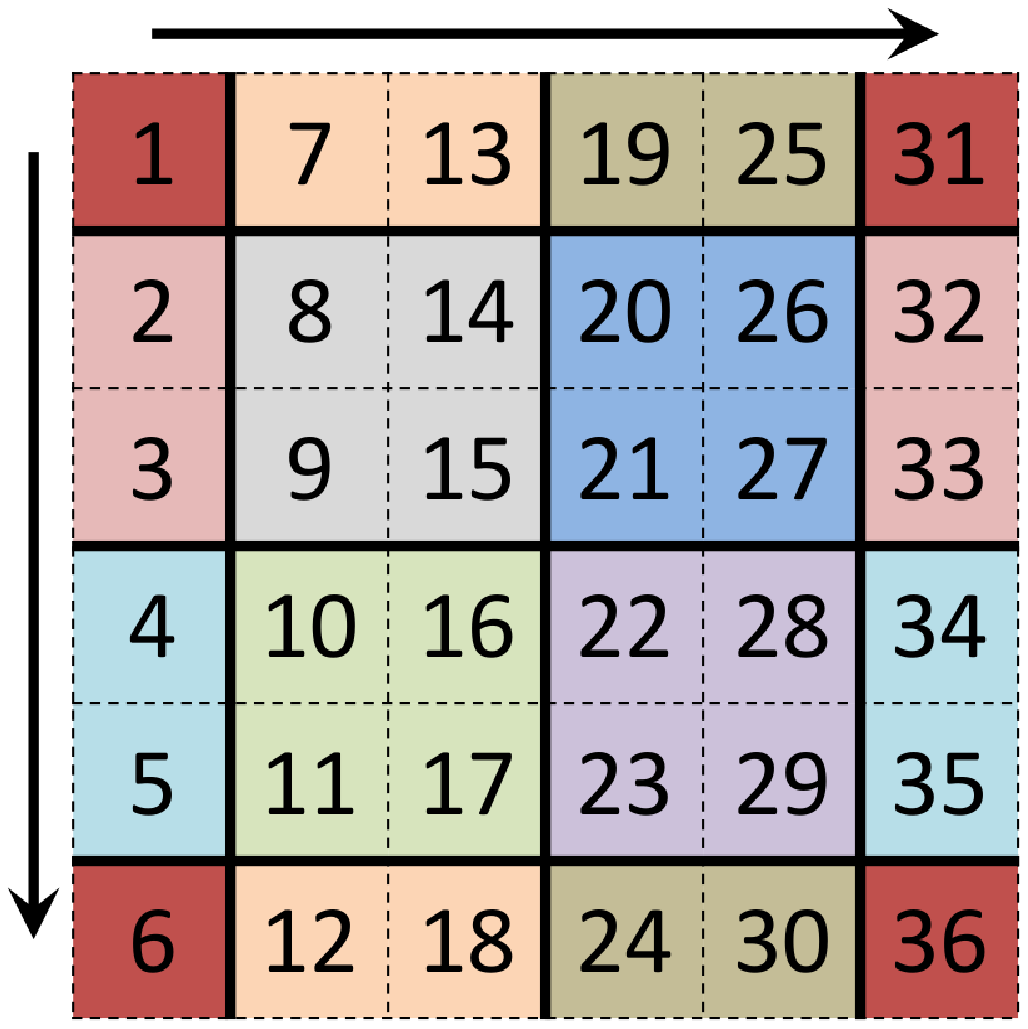, width = 1.5in}}}
\caption{Dividing $r \times c = 6 \times 6 $ input image (its pixels are numbered in $ 1,2,...,36 $) into $ \sqrt{n} \times \sqrt{n} = 2 \times 2 $ overlapping patches (solid squares in different colors). We assume a cycling processing of the patches (periodic boundary condition), thus there are $ N = 36 $ such patches, which can be divided into $ \{\Omega_j\}_{j=1}^{4} $ possible distinct groups of non-overlapping patches: (a) $ \Omega_1 $ -- without any shift, (b) $ \Omega_2 $ -- down shift, (c) $ \Omega_3 $ -- right shift, (d) $ \Omega_4 $ -- down \& right shift.}
\label{shiftPatches}
\end{figure}

{\em Proof of property 6 (\& 7): $ \|\W\|_2=1 $}. In order to denoise the image, we break it into $ \sqrt{n} \times \sqrt{n} $ overlapping patches. The following proof relies on the observation that we can divide the $ N $ \emph{overlapping} patches into $ \{\Omega_j\}_{j=1}^{n} $ distinct groups of \emph{non-overlapping} patches, as demonstrated in Figure \ref{shiftPatches}. As a consequence, the matrix $ \W $, which is a sum over $ N $ projection matrices (one per each patch), can be expressed as a sum over $ n $ distinct groups:
\begin{align}
\label{sumofProj}
\W & { = \frac{\mu}{\mu + n}\I + \frac{1}{\mu + n}\sum_{i=1}^{N} \left(\R_i^T \D_{s_i}\left(\D_{s_i}^{T}\D_{s_i}\right)^{-1}\D_{s_i}^{T} \R_i\right)} \\ \notag
   & { = \frac{\mu}{\mu + n}\I + \frac{1}{\mu + n}\sum_{i=1}^{N} \left(\R_i^T \hat{\Z}_i \R_i\right)} \\ \notag
   & { = \frac{\mu}{\mu + n}\I + \frac{1}{\mu + n}\sum_{j=1}^{n} \left[ \sum_{k\in \Omega_j} \left(\R_k^T \hat{\Z}_k \R_k\right)\right] } \\ \notag
   & { = \frac{\mu}{\mu + n}\I + \frac{1}{\mu + n}\sum_{j=1}^{n} \tilde{\W}_j}. \notag
\end{align}
Based on the property that $ \R_k\R_l^{T} = \textbf{0} $ for all $ (k \ne l) \in \Omega_j $, the following shows that $ \tilde{\W}_j $ is an idempotent matrix:
\begin{align}
\label{w2w}
(\tilde{\W}_j)^{2} & { = \left[ \sum_{k\in \Omega_j} (\R_k^T \hat{\Z}_k \R_k)\right]\left[ \sum_{l\in \Omega_j} (\R_l^T \hat{\Z}_l \R_l)\right] } \\ \notag
& { = \sum_{k\in \Omega_j} \left(\R_k^T \hat{\Z}_k \R_k\right)\left(\R_k^T \hat{\Z}_k \R_k\right) } \\ \notag
& { = \sum_{k\in \Omega_j} \left(\R_k^T \hat{\Z}_k\hat{\Z}_k \R_k\right) } \\ \notag
& { = \sum_{k\in \Omega_j} \left(\R_k^T \hat{\Z}_k \R_k\right) } \\ \notag
& { = \tilde{\W}_j},
\end{align}
where we have used the equality $ \R_k\R_k^T = \I\in\RR^{n\times n}$, and
\begin{align}
\label{z2z}
(\hat{\Z}_k)^{2} & { = \left(\D_{s_i}\left(\D_{s_i}^{T}\D_{s_i}\right)^{-1}\D_{s_i}^{T}\right)\left(\D_{s_i}\left(\D_{s_i}^{T}\D_{s_i}\right)^{-1}\D_{s_i}^{T}\right)} \\ \notag
& { = \D_{s_i}\left(\D_{s_i}^{T}\D_{s_i}\right)^{-1}\D_{s_i}^{T}} \\ \notag
& { = \hat{\Z}_k}.
\end{align}
As a result, following Equation (\ref{w2w}), we can infer that $ \| \tilde{\W}_j \|_2 = 1 $ \cite{horn2012matrix}. Finally, using the matrix-norm inequalities we get 
\begin{align}
\label{normW}
\|\W\|_2 & {   =  \| \frac{\mu}{\mu + n}\I + \frac{1}{\mu + n}\sum_{j=1}^{n} \tilde{\W}_j \|_2} \\ \notag
         & { \leq  \frac{\mu}{\mu + n}  \cdot \| \I \|_2 +  \frac{1}{\mu + n}  \cdot \sum_{j=1}^{n} \| \tilde{\W}_j \|_2 } \\ \notag
         & {  =      \frac{\mu}{\mu + n} + \frac{n}{\mu + n} } \\ \notag
         & {  = 1}. \notag
\end{align}
To conclude, based on the above and by relying on the property that $ 1 $ is an eigenvalue of $ \W $, we get that $ \lambda_{max}(\W) = 1 $, i.e, $ \|\W\|_2 = 1 $. \endproof

{\em Proof of property 8: $\|\W - \I\|_2\leq\frac{n}{\mu + n} $}. 
In general, the eigenvalues of $ \A+\b\I $ are bigger than the eigenvalues of $ \A $ by the constant $ \b $ \cite{horn2012matrix}. Therefore, the eigenvalues of 
$  (\W-\I) $ 
equal to $ \lambda(\W) - 1 $. Based on $ \frac{\mu}{\mu + n} \leq \lambda(\W) \leq 1 $, we get that
\begin{align}
\label{eigmaxwhat}
           \|\W-\I\|_2 & \leq |\lambda_{min}(\W) - 1| \\ \notag
                       & = \left| \frac{\mu}{\mu + n} - 1 \right| \\ \notag
                       & = \frac{n}{\mu + n}. \quad \endproof
\end{align}

\section{Equivalence between the SOS boosting and sharing the disagreement procedure}
\label{sos_and_disagree}
In the context of the K-SVD image denoising \cite{KSVD_REF1}, we show an equivalence between the SOS boosting recursive function (Equation (\ref{sos_algo})) and the disagreement and sharing approach (Algorithm \ref{consensus_algo}). The following study assumes fixed supports and dictionary during the iterations, i.e., the projection matrix $ \D_{s_i} $ of the $ i^{th} $ patch is known and fixed. In addition, a periodic boundary is considered (see Appendix \ref{periodic}), and we use the K-SVD matrix form (see Equation (\ref{sym1})) with $ \mu = 0 $, i.e.,
\begin{align}
\label{sym2}
\W & { = \frac{1}{n} \sum_{i=1}^{N} \R_i^T \D_{s_i}\left(\D_{s_i}^{T}\D_{s_i}\right)^{-1}\D_{s_i}^{T} \R_i } \\\notag
   & { = \frac{1}{n} \sum_{i=1}^{N} \R_i^T \hat{\Z}_i \R_i. } \notag
\end{align}

Following Algorithm \ref{consensus_algo}, we denote by $ \p_i^k $ the $ k^{th} $ iteration input patch to the denoising algorithm, which is influenced by the neighbors information. Using the projection matrix \mbox{$ \hat{\Z}_i = \D_{s_i}(\D_{s_i}^{T}\D_{s_i})^{-1}\D_{s_i}^{T} $}, let us compute the disagreement patch -- the difference between an independent denoised patch, as defined in Equation (\ref{clean_patch}),
\begin{align}
\label{disagree_p1}
\hat{\p}_i^{k} &= { \hat{\Z}_i\p_i^k },
\end{align}
and its corresponding patch from the global outcome (after patch averaging), $ \R_i\xh^{k} $, thus expressed by
\begin{align}
\label{disagree_2}
\q_i^{k} &= { \hat{\p}_i^{k} - \R_i\xh^{k}. }
\end{align}
Next, we subtract the disagreement patch from the corresponding noisy one, i.e.,
\begin{align}
\label{disagree_3}
\p_i^{k+1} &= { \R_i\y - \q_i^{k} } \\ \notag 
           &= { \R_i\y - \hat{\p}_i^{k} + \R_i\x^{k} } \\ \notag
           &= { \R_i\y - \hat{\Z}_i\p_i^{k} + \R_i\x^{k}. } \notag
\end{align}
Following Equation (\ref{disagree_p1}), the denoised version of $ \p_i^{k+1} $ is given by
\begin{align}
\label{disagree_4}
\hat{\p}_{i}^{k+1} &= { \hat{\Z}_i{\p}_i^{k+1} } \\ \notag
                   &= { \hat{\Z}_i\R_i\y - \hat{\Z}_i\hat{\Z}_i\p_i^{k} + \hat{\Z}_i\R_i\xh^{k} } \\ \notag
                   &= { \hat{\Z}_i\R_i\y - \hat{\Z}_i\p_i^{k} + \hat{\Z}_i\R_i\xh^{k} } \\ \notag
                   &= { \hat{\Z}_i\R_i\left(\y + \xh^{k}\right) - \hat{\p}_i^{k} } \notag         
\end{align}
where we use the idempotent property of $ \hat{\Z}_i $. Similarly to Equation (\ref{denosied_image_filtered}) and based on Equation (\ref{sym2}), the global denoised image is formulated by
\begin{align}
\label{disagree_5}
\xh^{k+1}  &= { \frac{1}{n} \sum_{i=1}^{N} {\R_i^{T}\hat{\p}_{i}^{k+1}} } \\ \notag
           &= { \frac{1}{n} \sum_{i=1}^{N} {\R_i^{T}\left( \hat{\Z}_i\R_i\left(\y + \xh^{k}\right) - \hat{\p}_i^{k} \right)}  } \\ \notag
           &= { \left( \frac{1}{n} \sum_{i=1}^{N} \R_i^{T}\hat{\Z}_i\R_i\right)\left(\y + \xh^{k}\right) - \frac{1}{n} \sum_{i=1}^{N} \R_i^{T}\hat{\p}_{i}^{k}  } \\ \notag
           &= { \W\left(\y + \xh^{k}\right) - \xh^{k}. } \notag
\end{align}
Thus, the SOS boosting (Equation (\ref{sos_algo})) and the ''sharing the disagreement'' algorithms are equivalent for a fixed $ \W $.

\section{Seeking for the fastest convergence}
\label{best_params}
\begin{figure}[tbp]
\centering
\mbox{{\epsfig{figure=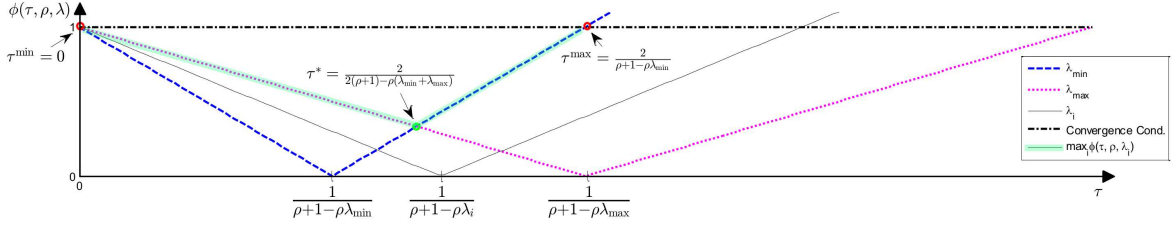, width = \linewidth}}}
\caption{Illustration of $ \phi(\tau,\rho,\lambda) $, the eigenvalues of the SOS error's transition matrix (in absolute value), as a function of $ \tau $. The dashed (blue), solid (black) and dotted (magenta) lines are corresponding to $ \phi $ with the arguments $ \lambda_{min}  $, $ \lambda_i $ and $ \lambda_{max} $, respectively. The horizontal black dash-dotted line denotes the condition for convergence, determining the values of $ \tau^{min} $ and $ \tau^{max} $ (red circles). The highlighted line illustrates the function $ \max_{i}{\phi(\tau, \rho, \lambda_i)} $, where $ \tau^{*} $ obtains its minimal value (green circle).}
\label{eig_tau_graph}
\end{figure}
We aim to provide conditions for the SOS algorithm to converge in terms of the parameters $ \rho $ and $ \tau $, and in addition get closed-form expression for $ \tau^{*} $, the solution of Equation (\ref{best_tau_rho}).
The eigenvalues of the SOS error's transition matrix (in absolute value) are formulated by
\begin{align}
\label{eigenval}
\phi(\tau, \rho, \lambda_i) & =  | \tau (\rho \lambda_i - \rho - 1)+1|,
\end{align}
where $ \{ \lambda_i \}_{i=1}^{N} $ are the eigenvalues of $ \W $. In the following analysis we shall assume that $\lambda_i \le 1$. Figure \ref{eig_tau_graph} plots $ \phi(\tau, \rho, \lambda_{min})$, $ \phi(\tau, \rho, \lambda_{i})$ and $ \phi(\tau, \rho, \lambda_{max}) $ as a function of $ \tau $, for $ \rho>0 $. As can been seen, $\phi $ has a negative slope for $ 0 \leq \tau \leq \frac{1}{\rho + 1 -\rho \lambda_i}$ and a positive slope for $ \tau > \frac{1}{\rho + 1 -\rho \lambda_i}$. All of these are true under the assumption that $ \rho \lambda_i - \rho - 1 < 0 $, which always holds for $ \lambda_i=1 $ and for 
\begin{align}
\label{cond_rho}
\rho > \rho^{min} = \min_i \frac{-1}{1-\lambda_i} \quad \forall \lambda_i \neq 1.
\end{align}

\begin{figure}[tbp]
\centering
\mbox{{\epsfig{figure=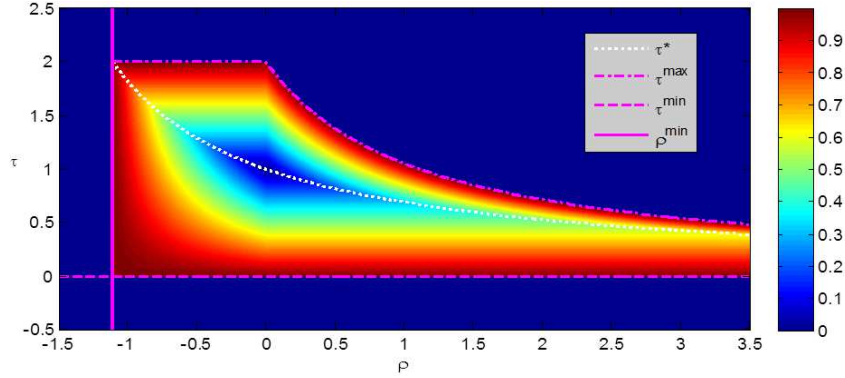, width = 4.5in}}}
\caption{Demonstration of $ \max_{i}{\phi(\tau, \rho, \lambda)} $ for $ \lambda \in [0.1,1] $, where cold and warm colors indicate small and large values, respectively. The magenta dash-dotted, dashed and solid lines plot the boundaries for convergence, corresponding to $ \tau^{max} $, $ \tau^{min} $ and $ \rho^{min} $, respectively. The white dotted line plots the analytic expression of $ \tau^{*} $ as a function of $ \rho $.}
\label{tau_rho_graph}
\end{figure}

Next, we shall find the valid range of $ \tau \in (\tau^{min},\tau^{max})$, satisfying $ \phi<1 $. Following Figure \ref{eig_tau_graph}, the minimal $ \tau $ that leads to an intersection with $ \phi = 1 $ is
\begin{align}
\label{cond_tau_min}
\tau^{min} = 0.
\end{align}
Then, by increasing $ \tau $ we get $ \phi<1 $, until reaching to $ \tau^{max} $ -- the first $ \tau>\tau^{min} $ that obtains $ \phi=1 $ again. As demonstrated in Figure \ref{eig_tau_graph}, for $ \rho>0 $, we get $ \tau^{max} = \frac{2}{\rho+1-\rho\lambda_{min}} $. While for $ \rho<0 $ we get $ \tau^{max} = \frac{2}{\rho+1-\rho\lambda_{max}} $, thus
\begin{align}
\label{cond_tau_max}
\tau^{max}=\min\left\lbrace \frac{2}{\rho+1-\rho\lambda_{min}} ,\frac{2}{\rho+1-\rho\lambda_{max}}\right\rbrace .
\end{align}
The obtained conditions on $ \rho $ and $ \tau $ are illustrated in Figure \ref{tau_rho_graph}, which plots the function $ \max_{i} \phi(\tau, \rho, \lambda_i) $ for $ \lambda \in [0.1,1] $, where $ \phi < 1$. As can be seen, $ \rho^{min} $, $ \tau^{min} $ and $ \tau^{max} $ bound perfectly the valid range for convergence.

We now turn to discuss $ \tau^{*} $, the solution of Equation (\ref{best_tau_rho}),
\begin{align}
\tau^{*} = \min_{\tau} \max_{1 \le i \le N} \phi(\tau, \rho, \lambda_i) \hspace{0.7em} \text{s.t.} \hspace{0.7em} \forall i  \hspace{0.5em} \phi(\tau, \rho, \lambda_i)<1. \notag
\end{align}
Following Figure \ref{eig_tau_graph}, under the obtained conditions on $ \rho $ and $ \tau $, we seek for minimum value of the highlighted graph (the green circle). As can be seen, $ \tau^{*} $ is obtained by the following equality
\begin{align}
\label{tau_star_0}
\phi(\tau^{*}, \rho, \lambda_{min}) = \phi(\tau^{*}, \rho, \lambda_{max}).
\end{align}
Based on condition (\ref{cond_rho}), we can infer that
\begin{align}
\label{tau_star_1}
\phi(\tau^{*}, \rho, \lambda_{min}) =  -\tau^{*} (\rho \lambda_{min} - \rho - 1)-1
\end{align}
and 
\begin{align}
\label{tau_star_2}
\phi(\tau^{*}, \rho, \lambda_{max}) =  \tau^{*} (\rho \lambda_{max} - \rho - 1)+1.
\end{align}
Substituting Equations (\ref{tau_star_1}) and (\ref{tau_star_2}) into Equation (\ref{tau_star_0}) lead to
\begin{align}
\tau^{*} =  \frac{2}{2(\rho+1)-\rho(\lambda_{min} + \lambda_{max})}. \notag
\end{align}
In addition, by substituting $ \tau^{*} $ into Equation (\ref{tau_star_2}) we get
\begin{align}
\gamma^{*} =  \frac{\rho(\lambda_{max}-\lambda_{min})}{2(\rho+1)-\rho(\lambda_{min} + \lambda_{max})}. \notag
\end{align}
 
An illustration that $ \tau^* $ obtains the minimal eigenvalue of the error's transition matrix is shown in Figure \ref{tau_rho_graph}, as a curve running through all possible values of $ \rho $. Note that the fastest convergence is obtained for the couple $ \tau=1$ and $ \rho=0 $ ($ \gamma^*=0 $), i.e., applying the original denoising algorithm only once, without any SOS step. However, as we aim to improve the denoising performance by strengthening the underlying signal (setting $ \rho \neq 0$), this choice of parameters is meaningless. 

\bibliographystyle{siam}  
\bibliography{docultexmm}

\end{document}